\documentclass{article}

\PassOptionsToPackage{numbers, compress}{natbib}

     \usepackage[final]{neurips_2019}

\usepackage[utf8]{inputenc} 
\usepackage[T1]{fontenc}    
\usepackage{hyperref}       
\usepackage{url}            
\usepackage{booktabs}       
\usepackage{amsfonts}       
\usepackage{nicefrac}       
\usepackage{microtype}      
\usepackage{amsmath}
\usepackage{algorithm}
\usepackage{algorithmic}
\usepackage{amsthm}
\newtheorem{definition}{Definition}
\newtheorem{statement}{Statement}
\usepackage{mathtools}

\usepackage{amsmath}
\usepackage{amsfonts}
\usepackage{amssymb}
\usepackage{wrapfig}
\usepackage{subcaption}
\usepackage{multirow}
 \usepackage{mathtools} 

\usepackage{verbatim}

\usepackage{anyfontsize}

\usepackage{color}
\usepackage{tikz}
\usetikzlibrary{arrows,shapes,snakes,automata,backgrounds,fit,petri}
\usepackage{adjustbox}

\newcommand{\RN}[1]{%
	\textup{\lowercase\expandafter{\it \romannumeral#1}}%
}

\makeatletter
\newcommand{\distas}[1]{\mathbin{\overset{#1}{\kern\z@\sim}}}%

\usepackage{enumitem}

\usepackage{algorithm}
\usepackage{algorithmic}

\newcommand{\beq}{\vspace{0mm}\begin{equation}}
\newcommand{\eeq}{\vspace{0mm}\end{equation}}
\newcommand{\beqs}{\vspace{0mm}\begin{eqnarray}}
\newcommand{\eeqs}{\vspace{0mm}\end{eqnarray}}
\newcommand{\barr}{\begin{array}}
\newcommand{\earr}{\end{array}}

\newcommand{\Z}{\mathbb{Z}}

\newcommand{\Ncal}{\mathcal{N}}

\newtheorem{theorem}{Theorem} 
\newtheorem{lemma}{Lemma}

\DeclareMathOperator{\RR}{\mathbb{R}} 

\ifx\assumption\undefined
\newtheorem{assumption}{Assumption}
\fi

\ifx\definition\undefined
\newtheorem{definition}{Definition}
\fi

\ifx\remark\undefined
\newtheorem{remark}{Remark}
\fi
\newcommand{\norm}[1]{\left\| #1 \right\|}

\title{Poisson-Minibatching for Gibbs Sampling with Convergence Rate Guarantees}

\newcommand{\Exv}[1]{\mathbf{E}\left[#1\right]}
\newcommand{\Abs}[1]{\left|#1\right|}

\author{%
  Ruqi Zhang \\
  Cornell University\\
  \texttt{rz297@cornell.edu} \\
   \And
   Christopher De Sa\\
   Cornell University \\
   \texttt{cdesa@cs.cornell.edu} \\
}

\begin{document}

\maketitle

\begin{abstract}
Gibbs sampling is a Markov chain Monte Carlo method that is often used for learning and inference on graphical models.
Minibatching, in which a small random subset of the graph is used at each iteration, can help make Gibbs sampling scale to large graphical models by reducing its computational cost.
In this paper, we propose a new auxiliary-variable minibatched Gibbs sampling method, {\it Poisson-minibatching Gibbs}, which both produces unbiased samples and has a theoretical guarantee on its convergence rate. 
In comparison to previous minibatched Gibbs algorithms, Poisson-minibatching Gibbs supports fast sampling from continuous state spaces and avoids the need for a Metropolis-Hastings correction on discrete state spaces.
We demonstrate the effectiveness of our method on multiple applications and in comparison with both plain Gibbs and previous minibatched methods.
\end{abstract}

\section{Introduction}
Gibbs sampling is a Markov chain Monte Carlo (MCMC) method which is widely used for inference on graphical models~\cite{koller2009probabilistic}. 
Gibbs sampling works by iteratively resampling a variable from its conditional distribution with the remaining variables fixed.
Although Gibbs sampling is a powerful method, its utility can be limited by its computational cost when the model is large. One way to address this is to use \emph{stochastic methods}, which use a subsample of the dataset or model---called a minibatch---to approximate the dataset or model used in an MCMC algorithm.
Minibatched variants of many classical MCMC algorithms have been explored \cite{welling2011bayesian,maclaurin2014firefly,de2018minibatch,li2017mini}, including the MIN-Gibbs algorithm for Gibbs sampling~\cite{de2018minibatch}.

In this paper, we propose a new minibatched variant of Gibbs sampling on factor graphs called {\it Poisson-minibatching Gibbs} (Poisson-Gibbs).
Like other minibatched MCMC methods, Poisson-minibatching Gibbs improves Gibbs sampling by reducing its computational cost.
In comparison to prior work, our method improves upon MIN-Gibbs in two ways.
First, it eliminates the need for a potentially expensive Metropolis-Hastings (M-H) acceptance step, giving it a better asymptotic per-iteration time complexity than MIN-Gibbs.
Poisson-minibatching Gibbs is able to do this by choosing a minibatch in a way that depends on the current state of the variables, rather than choosing one that is independent of the current state as is usually done in stochastic algorithms.
We show that such state-dependent minibatches can still be sampled quickly, and that an appropriately chosen state-dependent minibatch can result in a reversible Markov chain with the correct stationary distribution even without a Metropolis-Hastings correction step.

The second way that our method improves upon previous work is that it supports sampling over continuous state spaces, which are common in machine learning applications (in comparison, the previous work only supported sampling over discrete state spaces).
The main difficulty here for Gibbs sampling is that resampling a continuous-valued variable from its conditional distribution requires sampling from a continuous distribution, and this is a nontrivial task (as compared with a discrete random variable, which can be sampled from by explicitly computing its probability mass function).
Our approach is based on fast inverse transform sampling method, which works by approximating the probability density function (PDF) of a distribution with a polynomial~\cite{olver2013fast}.

In addition to these two new capabilities, we prove bounds on the convergence rate of Poisson-minibatching Gibbs in comparison to plain (i.e. not minibatched) Gibbs sampling.
These bounds can provide a recipe for how to set the minibatch size in order to come close to the convergence rate of plain Gibbs sampling.
If we set the minibatch size in this way, we can derive expressions for the per-iteration computational cost of our method compared with others; these bounds are summarized in Table~\ref{tab:cost}.
In summary, the contributions of this paper are as follows:
\begin{itemize}
    \item We introduce Poisson-minibatching Gibbs, a variant of Gibbs sampling which can reduce computational cost without adding bias or needing a Metropolis-Hastings correction step.
    \item We extend our method to sample from continuous-valued distributions.
    \item We prove bounds on the convergence rate of our algorithm, as measured by the spectral gap, on both discrete and continuous state spaces.
    \item We evaluate Poisson-minibatching Gibbs empirically, and show that its performance can match that of plain Gibbs sampling while using less computation at each iteration. 
\end{itemize}

\begin{table}[t!]
\begin{center}
\begin{tabular}{l l l}
\toprule
{\bf State Space} &{\bf Algorithm} &{\bf Computational Cost/Iter} \\
\midrule 
Discrete & Gibbs sampling & $O(D\Delta)$\\
& MIN-Gibbs \cite{de2018minibatch} & $O(D\Psi^2)$\\
& MGPMH \cite{de2018minibatch} & $O(DL^2 + \Delta)$ \\
& DoubleMIN-Gibbs \cite{de2018minibatch} & $O(DL^2 + \Psi^2)$ \\
& Poisson-Gibbs & $O(DL^2)$\\
\midrule
 Continuous& Gibbs with rejection sampling & $O(N\Delta)$\\
& PGITS: Poisson-Gibbs with ITS & $O(L^3)$\\
& PGDA: Poisson-Gibbs with double approximation & $O(L^2\log L)$\\
 \bottomrule
\end{tabular}
\end{center}
\caption{Computational complexity cost for a single-iteration of Gibbs sampling.
Here, $N$ is the required number of steps in rejection sampling to accept a sample, and the rest of the parameters are defined in Section~\ref{sec:pre}.}
\label{tab:cost}
\end{table}

\subsection{Preliminaries and Definitions}\label{sec:pre}
In this section, we present some background about Gibbs sampling and graphical models and give the definitions which will be used throughout the paper.
In this paper, we consider Gibbs sampling on a factor graph~\cite{koller2009probabilistic}, a type of graphical model that defines a probability distribution in terms of its \emph{factors}.
Explicitly, a factor graph consists of a set of variables $\mathcal{V}$ (each of which can take on values in some set $\mathcal{X}$) and a set of factors $\Phi$, and it defines a probability distribution $\pi$ over a state space $\Omega = \mathcal{X}^{\mathcal{V}}$, where the probability of some $x \in \Omega$ is
\[
\textstyle
  \pi(x) = \frac{1}{Z} \cdot \exp\left( \sum_{\phi\in\Phi} \phi(x) \right) = \frac{1}{Z} \cdot \prod_{\phi\in\Phi} \exp\left( \phi(x) \right).
\]
Here, $Z$ denotes the scalar factor necessary for $\pi$ to be a distribution.
Equivalently, we can think of this as the \emph{Gibbs measure} with energy function
\[
\textstyle
  U(x) = \sum_{\phi\in\Phi} \phi(x),
  \hspace{2em}\text{where}\hspace{2em}
  \pi(x) \propto \exp(U(x));
\]
this formulation will prove to be useful in many of the derivations later in the paper.
(Here, the $\propto$ notation denotes that the expression on the left is a distribution that is proportional to the expression on the right with the appropriate constant of proportionality to make it a distribution.)
In a factor graph, the factors $\phi$ typically only depend on a subset of the variables; we can represent this as a bipartite graph where the nodesets are $\mathcal{V}$ and $\Phi$ and where we draw an edge between a variable $i \in \mathcal{V}$ and a factor $\phi \in \Phi$ if $\phi$ depends on $i$.
For simplicity, in this paper we assume that the variables are indexed with natural numbers $\mathcal{V} = \{1, \ldots, n\}$.
We denote the set of factors that depend on the $i$th variable, as
\[
A[i] = \{\phi | \phi \text{ depends on variable } i, \ \phi\in \Phi\}
\]

\begin{wrapfigure}{R}{0.45\textwidth}
\centering
\begin{minipage}{0.45\textwidth}
\vspace{-4mm}
\begin{algorithm}[H]
   \caption{Gibbs Sampling}
   \label{alg:gibbs}
\begin{algorithmic}
   \STATE {\bfseries Input:} initial point $x$
   \LOOP
   \STATE \textbf{sample} variable $i \sim \operatorname{Unif}\{1, \ldots, n\}$
   \FORALL{$v \in \mathcal{X}$}
      \STATE $x(i) \leftarrow v$
      \STATE $U_v \leftarrow \sum_{\phi\in A[i]} \phi(x)$
   \ENDFOR
   \STATE \textbf{construct distribution} $\rho$ where \[ \textstyle \rho(v) \propto \exp(U_v) \]
   \STATE \textbf{sample} $v$ from $\rho$
   \STATE \textbf{update} $x(i) \leftarrow v$
   \STATE \textbf{output sample} $x$
   \ENDLOOP
\end{algorithmic}
\end{algorithm}
\end{minipage}
\end{wrapfigure}
An important property of a factor graph is that the conditional distribution of a variable can be computed using only the factors that depend on that variable.
This lends to a particularly efficient implementation of Gibbs sampling, in which only these adjacent factors are used at each iteration (rather than needing to evaluate the whole energy function $U$): this is illustrated in Algorithm~\ref{alg:gibbs}.

The performance of our algorithm will depend on several parameters of the graphical model, which we will now restate, from previous work on MIN-Gibbs~\cite{de2018minibatch}.
If the variables take on discrete values, we let $D = | \mathcal{X} |$ denote the number of values each can take on.
We let $\Delta = \max_i |A[i]|$ denote the maximum degree of the graph.
We assume that the magnitudes of the factor functions are all bounded, and for any $\phi$ we let $M_{\phi}$ denote this bound
\[
    \textstyle
    M_{\phi} = \left(\sup_{x \in \Omega} \phi(x) \right) - \left(\inf_{x \in \Omega} \phi(x) \right).
\]
Without loss of generality (and as was done in previous works~\cite{de2018minibatch}), we will assume that $0 \le \phi(x) \le M_{\phi}$ because we can always add a constant to any factor $\phi$ without changing the distribution $\pi$.
We define the \emph{local maximum energy} $L$ and \emph{total maximum energy} $\Psi$ of the graph as bounds on the sum of $M_{\phi}$ over the set of the factors associated with a single variable $i$ and the whole graph, respectively,
\[
\textstyle
L = \max_{i\in\{1, 2, \dots, N\}}\sum_{\phi \in A[i]} M_{\phi}
\hspace{2em}\text{and}\hspace{2em}
\Psi = \sum_{\phi \in \Phi} M_{\phi}.
\]
If the graph is very large and has many low-energy factors, the maximum energy of a graph can be much smaller than the maximum degree of the graph.
All runtime analyses in this paper assume that evaluating a factor $\phi$ and sampling from a small discrete distribution can be done in constant time.

\section{Poisson-Minibatching Gibbs Sampling}
In this section, we will introduce the idea of Poisson-minibatching under the setting in which we assume we can sample from the conditional distribution of $x(i)$ exactly.
One such example is when the state space of $x$ is discrete.
We will consider how to sample from the conditional distribution when exact sampling is impossible in the next section.

In plain Gibbs sampling, we have to compute the sum over all the factors in $A[i]$ to get the energy in every step.
When the graph is large, the computation of getting the energy can be expensive; for example, in the discrete case this cost is proportional to $D \Delta$.
The main idea of Poisson-minibatching is to augment a desired distribution with extra Poisson random variables, which control how and whether a factor is used in the minibatch for a particular iteration. \citet{maclaurin2014firefly} used a similar idea to control whether a data point will be included in the minibatch or not with augmented Bernoulli variables.
However, this method has been shown to be very inefficient when only updating a small fraction of Bernoulli variables in each iteration \cite{quiroz2016block}.
Our method does not suffer from the same issue due to the usage of Poisson variables which we will explain further later in this section. 

We define the conditional distribution of additional variable $s_{\phi}$ for each factor $\phi$ as 
\[
    \textstyle
  s_{\phi}|x \sim \text{Poisson}\left(\frac{\lambda M_{\phi}}{L} + \phi(x)\right)
\]
where $\lambda>0$ is a hyperparameter that controls the minibatch size.
Then the joint distribution of variables $x$ and $s$, where $s$ is a variable vector including all $s_{\phi}$, is $\pi(x, s) = \pi(x) \cdot \mathbf{P}(s|x)$ and so
\begin{equation}\label{eq:pi}
  \pi(x, s)
  \propto
  \exp\left(
  \sum_{\phi\in\Phi} \left(s_{\phi} \log\left( 1 + \frac{L}{\lambda M_{\phi}}\phi(x) \right)
   + s_{\phi}\log\left(\frac{\lambda M_{\phi}}{L}\right) - \log\left(  s_{\phi}! \right)\right)\right)
\end{equation}
Using (\ref{eq:pi}) allows us to compute conditional distributions (of the variables $x_i$) using only a subset of the factors.
This is because the factor $\phi$ will not contribute to the energy unless $s_{\phi}$ is greater than zero. If many $s_{\phi}$ are zero, then we only need to compute the energy over a small set of factors. Since
\[
\textstyle
\mathbf{E}\left[ \left| \{ \phi \in A[i] \mid s_{\phi} > 0 \} \right| \right]
\le
\mathbf{E}\left[\sum_{\phi\in A[i]} s_{\phi} \right] = \sum_{\phi\in A[i]} \left(\frac{\lambda M_{\phi}}{L} + \phi(x)\right)\leq \lambda + L,
\]
this implies that $\lambda + L$ is an upper bound of the expected number of non-zero $s_{\phi}$. When the graph is very large and has many low-energy factors, $\lambda + L$ can be much smaller than the factor set size, in which case only a small set of factors will contribute to the energy while most factor terms will disappear because $s_{\phi}$ is zero. 

Using Poisson auxiliary variables has two benefits. 
First, compared with the Bernoulli auxiliary variables as described in FlyMC~\cite{maclaurin2014firefly}, there is a simple method for sampling $n$ Poisson random variables in total expected time proportional to the sum of their parameters, which can be much smaller than $n$ \cite{de2018minibatch}.
This means that sampling $n$ Poisson variables can be much more efficient than sampling $n$ Bernoulli variables, which allows our method to avoid any inefficiencies caused by sampling Bernoulli variables as in FlyMC.
Second, compared with a fixed-minibatch-size method such as the one used in \cite{welling2011bayesian}, Poisson-minibatching has the important property that the variables $s_{\phi}$ are independent. Whether a factor will be contained in the minibatch is independent to each other. This property is necessary for proving convergence rate theorems in the paper. 

In Poisson-Gibbs, we will sample from the joint distribution alternately.
At each iteration we can (1) first re-sample all the $s_{\phi}$, then (2) choose a variable index $i$ and re-sample $x(i)$.
Here, we can reduce the state back to only $x$, since the future distribution never depends on the current value of $s$. Essentially, we only bother to re-sample  the $s_{\phi}$ on which our eventual re-sampling of $x(i)$ depends: statistically, this is equivalent to re-sampling all $s_{\phi}$.
Doing this corresponds to Algorithm~\ref{alg:poisson-gibbs}.

However, minibatching by itself does not mean that the method must be more effective than plain Gibbs sampling.
It is possible that the convergence rate of the minibatched chain becomes much slower than the original rate, such that the total cost of the minibatch method is larger than that of the baseline method even if the cost of each step is smaller.
To rule out this undesirable situation, we prove that the convergence speed of our chain is not slowed down, or at least not too much, after applying minibatching.
To do this, we bound the convergence rate of our algorithm, as measured by the \emph{spectral gap}~\cite{levin2017markov}, which is the gap between the largest and second-largest eigenvalues of the chain's transition operator. 
This gap has been used previously to measure the convergence rate of minibatched MCMC~\cite{de2018minibatch}.

\begin{theorem}\label{thm:discrete}
Poisson-Gibbs (Algorithm~\ref{alg:poisson-gibbs}) is reversible and has a stationary distribution $\pi$.
Let $\bar{\gamma}$ denote its spectral gap, and let $\gamma$ denote the spectral gap of plain Gibbs sampling. If we use a minibatch size parameter $\lambda \ge 2 L$, then
\[
\bar{\gamma}
\ge
\exp\left( - \frac{4L^2}{\lambda} \right) \cdot\gamma.
\]
\end{theorem}

This theorem guarantees that the convergence rate of Poisson-Gibbs will not be slowed down by more than a factor of $\exp( - 4L^2 / \lambda )$. If we set $\lambda = \Theta(L^2)$, then this factor becomes $O(1)$, which is independent of the size of the problem. We proved Theorem~\ref{thm:discrete} and the other theorems in this paper using the technique of Dirichlet forms, which is a standard way of comparing the spectral gaps of two chains by comparing their transition probabilities (more details are in the supplemental material).

Next, we derive expressions for the overall computational cost of Algorithm~\ref{alg:poisson-gibbs}, supposing that we set $\lambda = \Theta(L^2)$ as suggested by Theorem~\ref{thm:discrete}.
First, we need to evaluate the cost of sampling all the Poisson-distributed $s_{\phi}$.
While a na\"ive approach to sample this would take $O(\Delta)$ time, we can do it substantially faster.
For brevity, and because much of the technique is already described in the previous work~\cite{de2018minibatch}, we defer an explicit analysis to the supplementary material, and just state the following.
\begin{statement}\label{stmt:samplingphi}
Sampling all the auxiliary variables $s_{\phi}$ for $\phi \in A[i]$ can be done in average time $O(\lambda + L)$, resulting in a sparse vector $s_{\phi}$.
\end{statement}
Now, to get an overall cost when assuming exact sampling from the conditional distribution, we consider discrete state spaces, in which we can sample from the conditional distribution of $x(i)$ exactly.
In this case, the cost of a single iteration of Poisson-Gibbs will be dominated by the loop over $v$. 
This loop will run $D$ times, and each iteration will take $O(|S|)$ time to run.
On average, this gives us an overall runtime $O((\lambda + L) \cdot D) = O(L^2 D)$ for Poisson-Gibbs.
Note that due to the fast way we sample Poisson variables, the cost of sampling Poisson variables is negligible compared to other costs. 

In comparison, the cost of the previous algorithms MIN-Gibbs, MGPMH and DoubleMIN-Gibbs~\cite{de2018minibatch} are all larger in big-$O$ than that of Poisson-Gibbs, as showed in Table~\ref{tab:cost}.
MGPMH and DoubleMIN-Gibbs need to conduct an M-H correction, which adds to the cost, and the cost of MIN-Gibbs and DoubleMIN-Gibbs depend on $\Psi$ which is a global statistic.
By contrast, our method does not need additional M-H step and is not dependent on global statistics. 
Thus the total cost of Gibbs sampling can be reduced more by Poisson-minibatching compared to the previous methods. 

\paragraph{Application of Poisson-Minibatching to Metropolis-Hastings.}
Poisson-minibatching method can be applied to other MCMC methods, not just Gibbs sampling. To illustrate the general applicability of Poisson-minibatching method, we applied Poisson-minibatching to Metropolis-Hasting sampling and call it \emph{Poisson-MH} (details of this algorithm and a demonstration on a mixture of Gaussians are given in the supplemental material).
We get the following convergence rate bound.
\begin{theorem}\label{thm:mh}
  Poisson-MH is reversible and has a stationary distribution $\pi$.
  If we let $\bar{\gamma}$ denote its spectral gap, and let $\bar{\gamma}$ denote the spectral gap of plain M-H sampling with the same proposal and target distributions, then
  \[
    \textstyle
    \bar{\gamma}
    \ge
    \frac{1}{2}\exp\left( - \frac{L^2}{\lambda + L}\right)\cdot\gamma.
  \]
\end{theorem}

\section{Poisson-Gibbs on Continuous State Spaces}
In this section, we consider how to sample from a continuous conditional distribution, i.e. when $\mathcal{X} = [a,b] \subset \mathbb{R}$, without sacrificing the benefits of Poisson-minibatching. 
The main difficulty is that sampling from an arbitrary continuous conditional distribution is not trivial in the same way as sampling from an arbitrary discrete conditional distribution is.
Some additional sampling method is required. In principle, we can combine any sampling method with Poisson-minibatching, such as rejection sampling which is commonly used in Gibbs sampling.
However, rejection sampling needs to evaluate the energy multiple times per sample, so even if we reduce the cost of evaluating the energy by minibatching, the total cost can still be large, besides which there is no good guarantee on the convergence rate of rejection sampling. 

In order to sample from the conditional distribution efficiently, we propose a new sampling method based on inverse transform sampling (ITS) method. The main idea is to approximate the continuous distribution with a polynomial; this requires only a number of energy function evaluations proportional to the degree of the polynomial. We provide overall cost and theoretical analysis of convergence rate for our method.
\begin{figure}[t]
\begin{minipage}[t]{6.4cm}
\begin{algorithm}[H]
  \caption{Poisson-Gibbs}
  \begin{algorithmic}
    \label{alg:poisson-gibbs}
    \STATE \textbf{given:} initial state $x \in \Omega$
    \LOOP
      \STATE \textbf{sample} variable $i \sim \operatorname{Unif}\{1, \ldots, n\}$.
      \FORALL{$\phi$ \textbf{in} $A[i]$}
        \STATE \textbf{sample}  $s_{\phi} \sim \text{Poisson}\left( \frac{\lambda M_{\phi}}{L} + \phi(x) \right)$
      \ENDFOR
      \STATE $S \leftarrow \{ \phi | s_{\phi} > 0 \}$
   \FORALL{$v \in \mathcal{X}$}
      \STATE $x(i) \leftarrow v$
      \STATE $U_v \leftarrow \sum_{\phi \in S} s_{\phi} \log\left( 1 + \frac{L}{\lambda M_{\phi}}\phi(x) \right)$
   \ENDFOR
      \STATE \textbf{construct distribution} $\rho$ where
      \[
        \rho(v) \propto \exp(U_v)
      \]
      \STATE \textbf{sample} $v$ from $\rho$
      \STATE \textbf{update} $x(i) \leftarrow v$
      \STATE \textbf{output sample} $x$
    \ENDLOOP
  \end{algorithmic}
\end{algorithm}
\end{minipage}
\hfill
\begin{minipage}[t]{7.3cm}
\begin{algorithm}[H]
  \caption{PGDA: Poisson-Gibbs Double Chebyshev Approximation}
  \begin{algorithmic}
    \label{alg:PGDA}
    \STATE \textbf{given:} state $x \in \Omega$, degree $m$ and $k$, domain $[a,b]$
    \LOOP
      \STATE \textbf{set} $i$, $s_{\phi}$, $S$, and $U$ as in Algorithm~\ref{alg:poisson-gibbs}.
      \STATE \textbf{construct} degree-$m$ Chebyshev polynomial approximation of energy $U_v$ on $[a,b]$: $\tilde{U}_v$
      \STATE \textbf{construct} degree-$k$ Chebyshev polynomial approximation:$\tilde{f}(v) \approx \exp(\tilde U_v)$
      \STATE \textbf{compute} the CDF polynomial
      \[
        \tilde{F}(v) = \left( \int_a^b \tilde f(y) \; dy \right)^{-1} \int_a^v \tilde f(y) \; dy 
      \]
      \STATE \textbf{sample} $u \sim \operatorname{Unif}[0,1]$.
      \STATE \textbf{solve} root-finding problem for $v$: $\tilde{F}(v) = u$
       \STATE $\triangleright$ Metropolis-Hastings correction:
    \STATE $p \leftarrow \frac{\exp(U_{v})\tilde{f}(x(i))}{\exp(U_{x(i)})\tilde{f}(v)}$
    \STATE \textbf{with probability} $\min(1,p)$, set $x(i) \leftarrow v$ \STATE \textbf{output sample} $x$
    \ENDLOOP
  \end{algorithmic}
\end{algorithm}
\end{minipage}
\end{figure}

\paragraph{Poisson-Gibbs with Double Chebyshev Approximation.}
Inverse transform sampling is a classical method that generates samples from a uniform distribution and then transforms them by the inverse of cumulative distribution function (CDF) of the desired distribution. 
Since the CDF is often intractable in practice, Fast Inverse Transform Sampling (FITS)~\cite{olver2013fast} uses a Chebyshev polynomial approximation to estimate the PDF fast and then get the CDF by computing an integral of a polynomial. Inspired by FITS, we propose Poisson-Gibbs with double Chebyshev approximation (PGDA). 

The main idea of double Chebyshev approximation is to approximate the energy function first and then the PDF by using Chebyshev approximation \emph{twice}. Specifically, we first get a polynomial approximation to the energy function $U$ on $[a,b]$, denoted by $\tilde{U}$, the \emph{Chebyshev interpolant}~\cite{trefethen2013approximation}
\begin{align}\label{eq:energy}
\tilde{U}(x) = \sum_{k=0}^m \alpha_k T_k\left(\frac{2(x - a)}{b-a} - 1\right),\ \alpha_k\in \RR,\ x\in[a, b],
\end{align}
where $T_k (x) = \cos(k \cos^{-1} x)$ is the degree-$k$ Chebyshev polynomial.
Although the domain is continuous, we only need to evaluate $U$ on $m+1$ Chebyshev nodes to construct the interpolant, and the expansion coefficients $\alpha_k$ can be computed stably in $O(m \log m)$ time. 
The following theorem shows that the error of a Chebyshev approximation can be made arbitrarily small with large $m$.
(Although stated for the case of $[a,b] = [-1,1]$, it easily generalizes to arbitrary $[a,b]$.)

\begin{theorem}[Theorem 8.2 from~\citet{trefethen2013approximation}]
\label{thm:cheby}
Assume $U$ is analytic in the
open Bernstein ellipse $B([-1, 1], \rho)$, where the Bernstein ellipse is a region in the complex plane bounded by an ellipse with foci at $\pm 1$ and semimajor-plus-semiminor axis length $\rho > 1$. If for all $x \in B([-1, 1], \rho)$, $|U(x)| \leq V$ for some
constant $V > 0$, the error of the Chebyshev interpolant on $[-1,1]$ is bounded by
\begin{align*}
| \tilde{U}(x) - U(x) | \leq \delta_m
\hspace{2em}\text{where}\hspace{2em}
\delta_m = \frac{4V\rho^{-m}}{\rho - 1}.
\end{align*}
\end{theorem}

After getting the approximation of the energy, we can get the PDF by $\exp(\tilde{U})$. However, it is generally hard to get the CDF now since the integral of $\exp(\tilde{U})$ for polynomial $\tilde U$ is usually intractable.
So, we use \emph{another} Chebyshev approximation $\tilde{f}$ to estimate $\exp(\tilde{U})$.
Constructing the second Chebyshev approximation requires no additional evaluations of energy functions; its total computational cost is $\tilde O(mk)$ because we need to evaluate a degree-$m$ polynomial $k$ times to compute the coefficients.
After doing this, we are able to compute the CDF directly since it is the integral of a polynomial. With the CDF $\tilde{F}(x)$ in hand, inverse transform sampling is used to generate samples. First, a pseudo-random sample $u$ is generated from the uniform distribution on $[0,1]$, and then we solve the following root-finding problem for $x$: $\tilde{F}(x) = u$.
Since $\tilde{F}(x)$ is a polynomial, this root-finding problem can be solved by many standard methods. We use bisection method to ensure the robustness of the algorithm~\cite{olver2013fast}. 

Importantly, the sample we get here is actually from an \emph{approximation} of the CDF.
To correct the error introduced by the polynomial approximation, we add a M-H correction as the final step to make sure the samples come from the target distribution. Our algorithm is given in Algorithm~\ref{alg:PGDA}. As before, we prove a bound on PGDA in terms of the spectral gap, given the additional assumption that the factors $\phi$ are analytic.

\begin{theorem}\label{thm:doupoly}
  PGDA (Algorithm~\ref{alg:PGDA}) is reversible and has a stationary distribution $\pi$.
  Let $\bar{\gamma}$ denote its spectral gap, and let $\gamma$ denote the spectral gap of plain Gibbs sampling. Assume $\rho > 1$ is some constant such that every factor function $\phi$, treated as a function of any single variable $x_i$, must
   be analytically continuable to the Bernstein ellipse with radius parameter $\rho$ shifted-and-scaled so that its foci are at $a$ and $b$, such that it satisfies $| \phi(z) | \le M_{\phi}$ anywhere in that ellipse.
  Then, if $\lambda \log(2) \ge 4 L$, and if $m$ is set large enough that $4 \rho^{-m/2} \le \sqrt{\rho} - 1$, then it will hold that
  \begin{align*}
     \bar{\gamma}
    &\geq
    \left(1 - 4\sqrt{\digamma}\right)\exp\left(\frac{-4L^2}{\lambda}\right)\cdot \gamma,
    \hspace{1em}\text{where}\hspace{1em}
    \digamma
    =
    \frac{4\cdot \exp\left(8 L \right)\cdot\rho^{-\frac{k}{2}}}{\sqrt{\rho} - 1}
  + \exp\left( \frac{16 L \cdot \rho^{-\frac{m}{2}}}{\sqrt{\rho} - 1} \right) - 1.
    \end{align*}
\end{theorem}
Similar to Theorem~\ref{thm:discrete}, this theorem implies that the convergence rate of PGDA can be slowed down by at most a constant factor relative to plain Gibbs. If we set $m = \Theta(\log L)$, $k = \Theta(L)$ and $\lambda = \Theta(L^2)$, then the ratio of the spectral gaps will also be $O(1)$, which is independent of the problem parameters.
Note that it is possible to combine FITS with Poisson-Gibbs directly (i.e. use only one polynomial approximation to estimate the PDF directly), and we call this method {\it Poisson-Gibbs with fast inverse transform sampling} (PGITS). It turns out that PGDA is more efficient than PGITS since PGDA requires fewer evaluations of $U$ to achieve the same convergence rate.
If we set the parameters as above, the total computational cost of PGDA is $O(m\cdot (\lambda + L) + m\cdot k) = O(\log L \cdot (L^2 + L)) = O(\log L\cdot L^2)$. On the other hand, the cost of PGITS to achieve the same constant-factor spectral gap ratio is $O(L^3)$. A derivation of this is given in the supplemental material.

\section{Experiments}
We demonstrate our methods on three tasks including Potts models, continuous spin models and truncated Gaussian mixture in comparison with plain Gibbs sampling and previous minibatched Gibbs sampling. We release the code at \url{https://github.com/ruqizhang/poisson-gibbs}.
\subsection{Potts Models}\label{sec:ising}
\begin{figure*}[t!]
    \centering
    \begin{tabular}{cccc}		
    	\includegraphics[width=4.2cm]{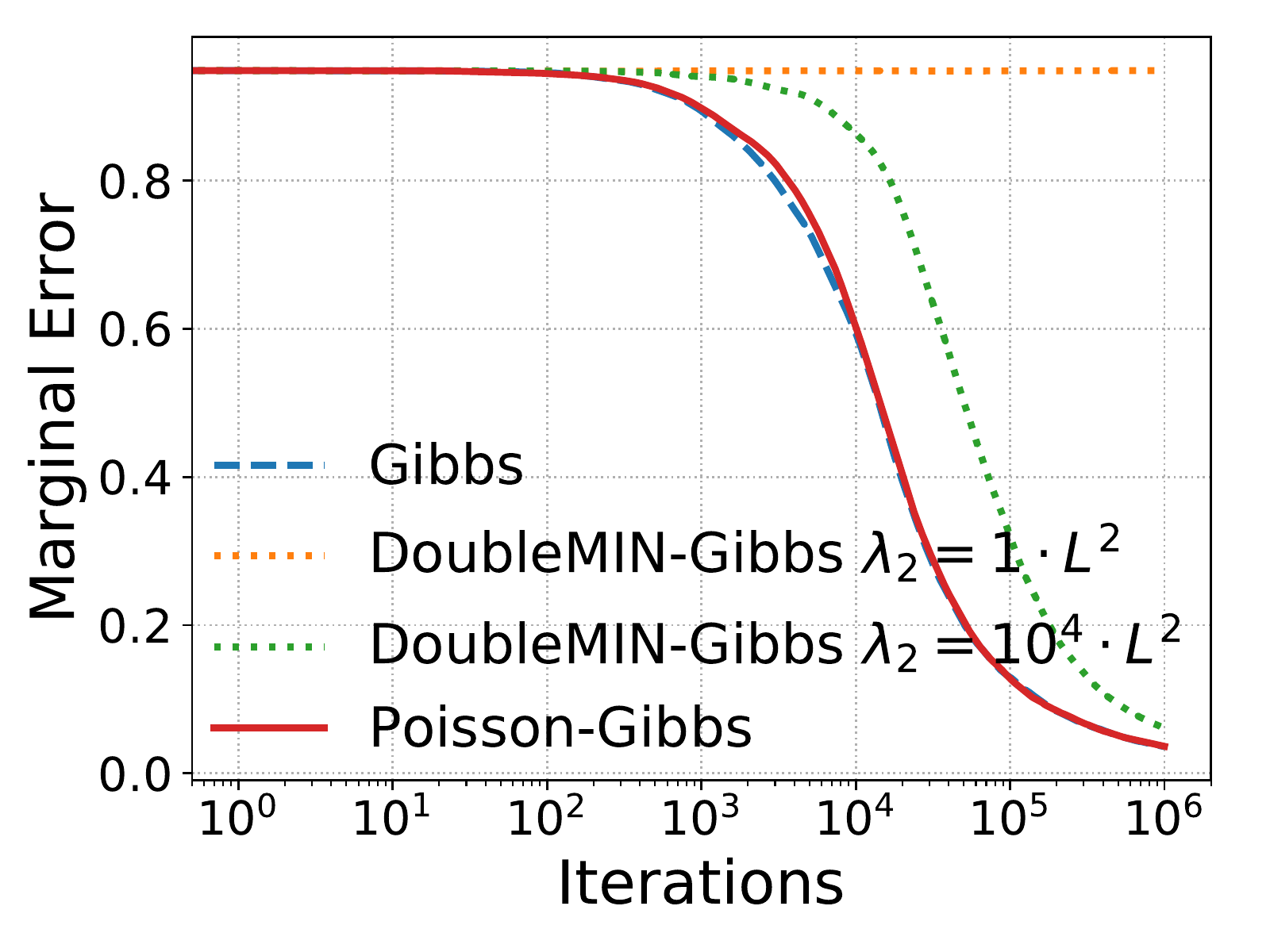}  &
    	\includegraphics[width=4.2cm]{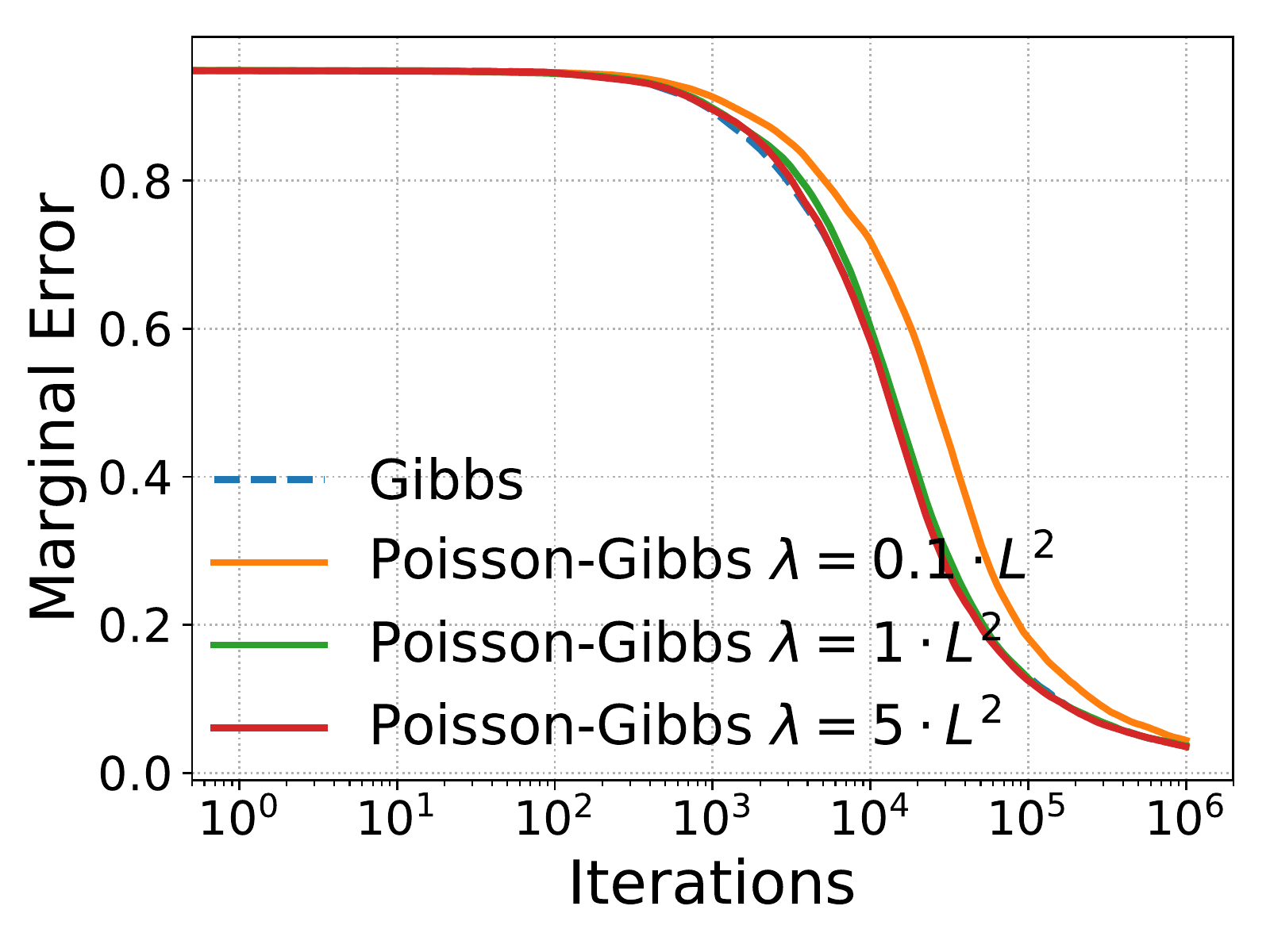} &
    	\includegraphics[width=4.2cm]{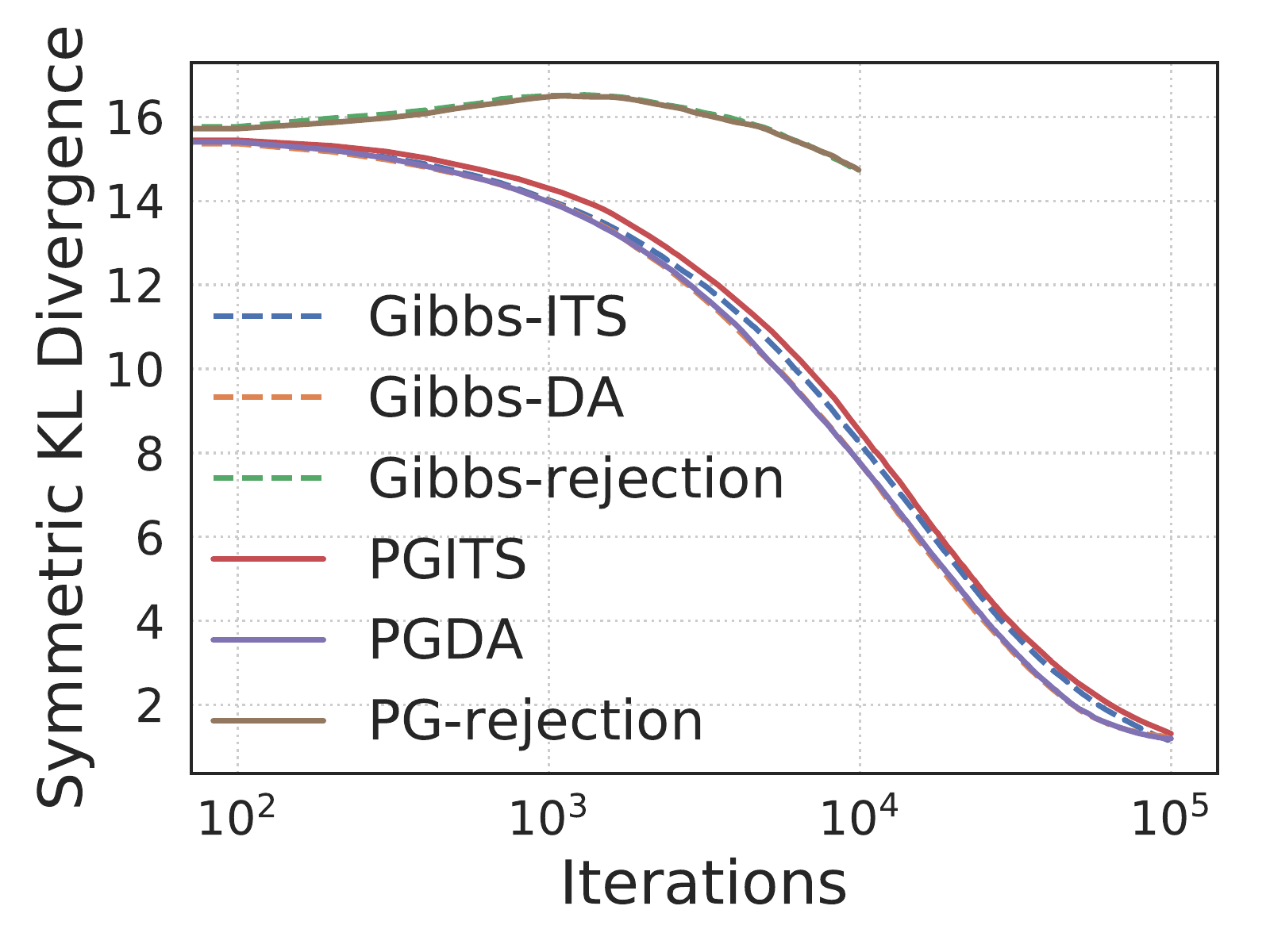} 
    	\\		
    	(a) &
    	(b) &
    	(c)
    	\hspace{-0mm}\\		
    \end{tabular}
    \caption{(a) Marginal error comparison among Poisson-Gibbs and previous methods on a Potts model. (b) Marginal error of Poisson-Gibbs on varying values of $\lambda$ on a Potts model. (c) Symmetric KL divergence comparison among PGITS, PGDA and previous methods on a continuous spin model.}
    \label{fig:ising}
\end{figure*}
\begin{figure*}[t!]
    \centering
    \begin{tabular}{cccc}		
    	\includegraphics[width=4.2cm]{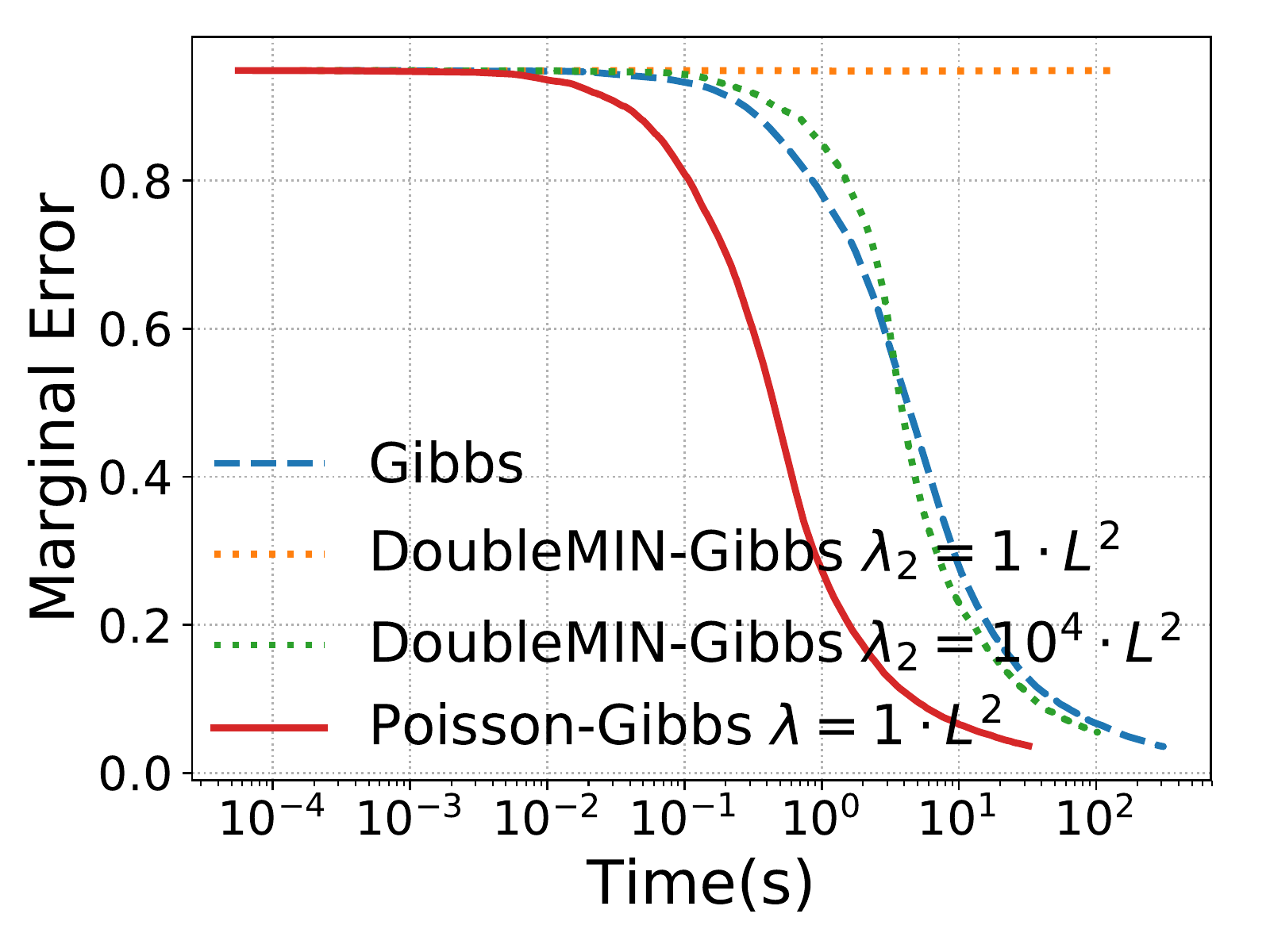}  &
    	\includegraphics[width=4.2cm]{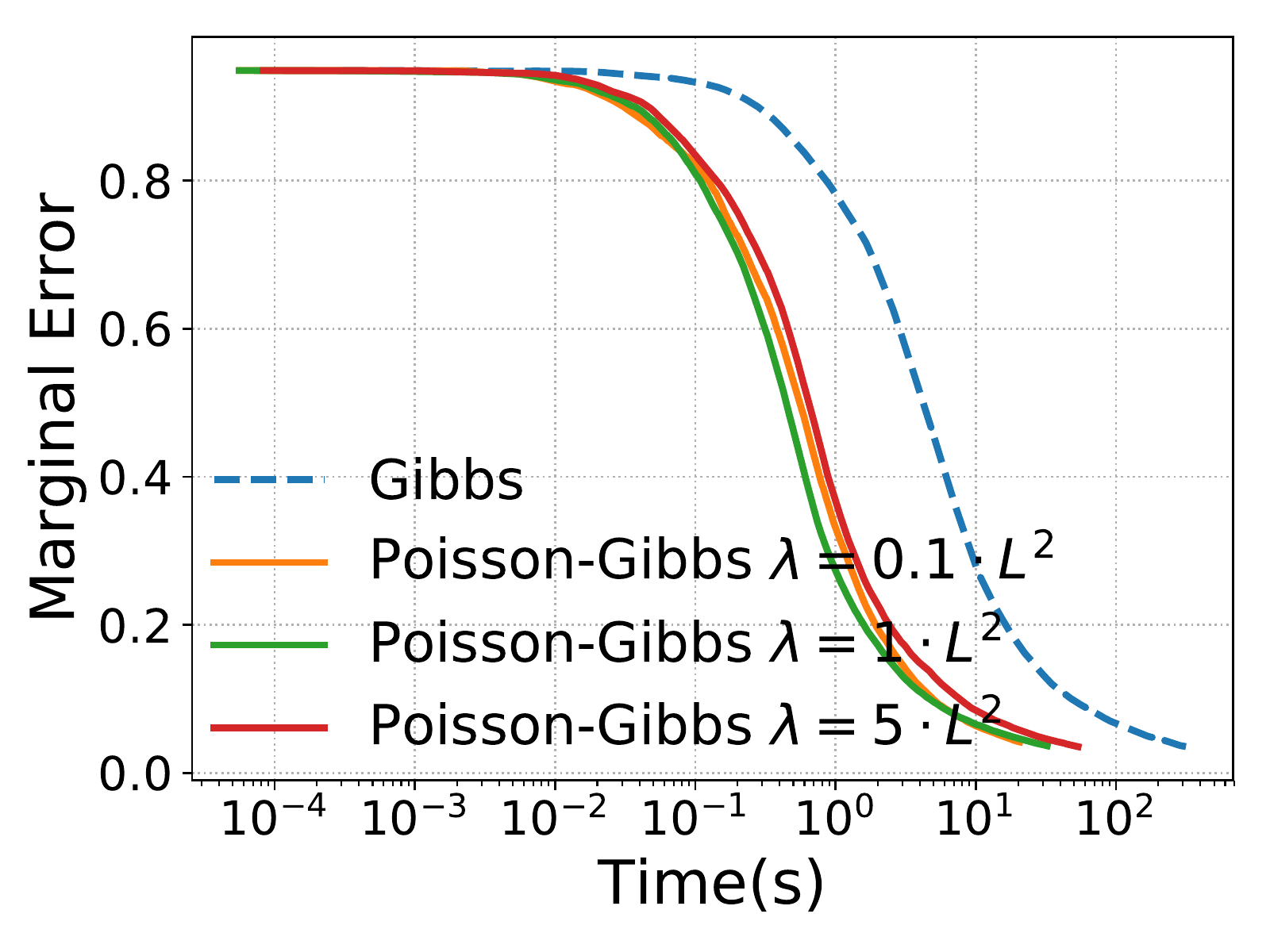} &
    	\includegraphics[width=4.2cm]{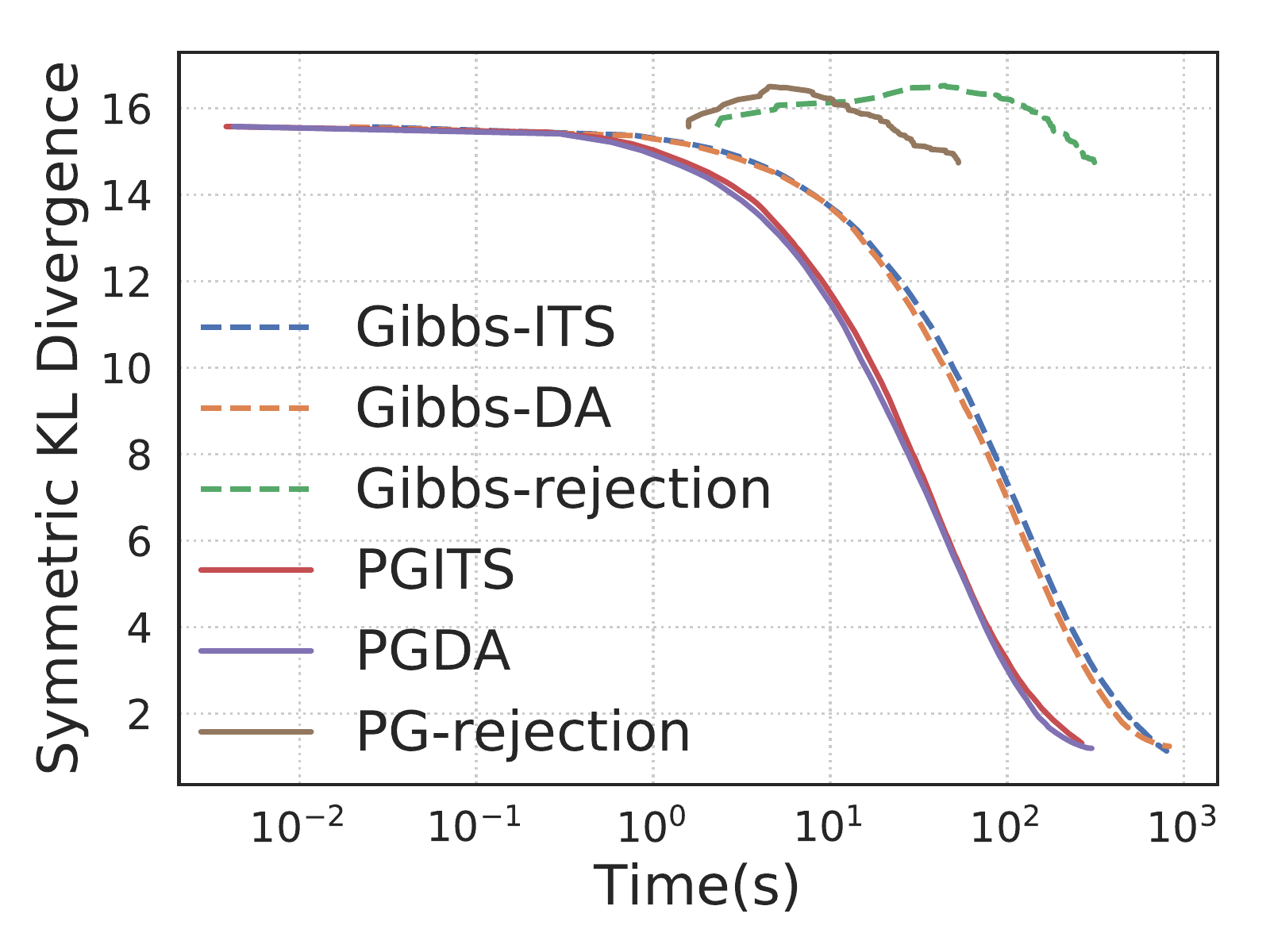} 
    	\\		
    	(a) &
    	(b) &
    	(c)
    	\hspace{-0mm}\\		
    \end{tabular}
    \caption{Runtime comparisons with the same experimental setting as in Figure \ref{fig:ising}.}
    \label{fig:ising-time}
\end{figure*}

We first test the performance of Poisson-minibatching Gibbs sampling on the Potts model \cite{potts1952some} as in \citet{de2018minibatch}. The Potts model is a generalization of the Ising model \cite{ising1925beitrag} with domain $\{1,\ldots, D\}$ over an $N\times N$ lattice. The energy of a configuration is the following:
\[
U(x) = \sum_{i=1}^n \sum_{j=1}^n\beta \cdot A_{ij} \cdot \delta\left(x(i), x(j)\right)
\]
where the $\delta$ function equals one only when $x(i) = x(j)$ and zero otherwise. $A_{ij}$ is the interaction between two sites $i$ and $j$ and $\beta$ is the inverse temperature. As was done in previous work, we set the model to be fully connected and the interaction $A_{ij}$ is determined by the distance between site $i$ and site $j$ based on a Gaussian kernel~\cite{de2018minibatch}. The graph has $n = N^2 = 400$ variables in total, $\beta = 4.6$ and $D=10$. On this model, $L = 5.09$.

We first compare our method with two other methods: plain Gibbs sampling and the most efficient MIN-Gibbs methods on this task, DoubleMIN-Gibbs. Note that, in comparison to our method, DoubleMIN-Gibbs needs an additional M-H correction step which requires a second minibatch to be sampled. 
We set $\lambda = 1\cdot L^2$ for all minibatch methods.
We tried two values for the second minibatch size in DoubleMIN-Gibbs $\lambda_2 = 1\cdot L^2$ and $10^4\cdot L^2$.
We compute run-average marginal distributions for each variable by collecting samples. By symmetry, the marginal for each variable in the stationary distribution is uniform, so the $\ell_2$-distance between the estimated marginals and the uniform distribution can be used to evaluate the convergence of Markov chain. We report this  marginal error averaged over three runs.

Figure~\ref{fig:ising}a shows the $\ell_2$-distance marginal error as a function of iterations. We observe that Poisson-Gibbs performs comparably with plain Gibbs and it outperforms DoubleMIN-Gibbs significantly especially when $\lambda_2$ is not large enough. The performance of DoubleMIN-Gibbs is highly influenced by the size of the second minibatch. We have to increase the second minibatch to $10^4 \cdot L^2$ in order to make it converge. This is because the variance of M-H correction will be very large when the second minibatch is not large enough. On the other hand, Poisson-Gibbs does not require an additional M-H correction which not only reduces the computational cost but also improves stability.
In Figure~\ref{fig:ising}b, we show the performance of our method with different values of $\lambda$. When we increase the minibatch size, the convergence speed of Poisson-Gibbs approaches plain Gibbs, which validates our theory. The number of factors being evaluated of Poisson-Gibbs varies each iteration, thus we report the average number which are 7, 28 and 132 respectively for $\lambda=0.1\cdot L^2$, $1\cdot L^2$ and $5\cdot L^2$. 

The runtime comparisons with the same setup are reported in Figure \ref{fig:ising-time}a and 2b to demonstrate the computational speed-up of Poisson-Gibbs empirically. We can see that the results align with our theoretical analysis: Poisson-Gibbs is significantly faster than plain Gibbs samping and faster than previous minibatched Gibbs sampling methods. Compared to plain Gibbs, Poisson-Gibbs speeds up the computation by evaluating only a subset of factors in each iteration. Compared to DoubleMIN-Gibbs, Poisson-Gibbs is faster because it removes the need of an additional M-H correction step. 

\subsection{Continuous Spin Models}
In this section, we study a more general setting of spin models where spins can take continuous values.
Continuous spin models are of interest in both the statistics and physics communities~\cite{michel2015event,bruce1985universality,dommers2017continuous}. 
This random graph model can also be used to describe complex networks such as social, information, and biological networks~\cite{newman2003structure}.
We consider the energy of a configuration as the following:
\[
U(x) = \sum_{i=1}^n \sum_{j=1}^n\beta \cdot A_{ij} \cdot \left(x(i) \cdot x(j) + 1\right)
\]
where $x(i)\in [0, 1]$ and $\beta=1$. Notice that the existing minibatched Gibbs sampling methods \cite{de2018minibatch} are not applicable on this task since they can be used only on discrete state spaces. We compare PGITS, PGDA with: (1) Gibbs sampling with FITS (Gibbs-ITS); (2) Gibbs sampling with Double Chebyshev approximation (Gibbs-DA); (3) Gibbs with rejection sampling (Gibbs-rejection); and (4) Poisson-Gibbs with rejection sampling (PG-rejection). We use symmetric KL divergence to quantitatively evaluate the convergence. On this model, $L=13.71$ and we set $\lambda = L^2$. The degree of polynomial is $m=3$ for PGITS and the first approximation in PGDA. The degree of polynomial is $k=10$ for the second approximation in PGDA.
In rejection sampling, we set the proposal distribution to be $w g$ where $g$ is the uniform distribution on $[0,1]$ and $w$ is a constant tuned for best performance. The ground truth stationary distribution is obtained by running Gibbs-ITS for $10^7$ iterations.

On this task, the average number of evaluated factors per iteration of Poisson-Gibbs is 190. Figure~\ref{fig:ising}c shows the symmetric KL divergence as a function of iterations, with results averaged over three runs. Observe that our methods achieve comparable performance to Gibbs sampling with only a fraction of factors. For rejection sampling, the average steps needed for a sample to be accepted is greater than 300 which means that the cost is much larger than that of PGITS and PGDA. Given the same time budget, it can only run for many fewer iterations (we run it for $10^4$ iterations). On the other hand, the two Chebyshebv approximation methods are much more efficient for both Poisson-Gibbs and plain Gibbs. The advantage of FITS over rejection sampling has also been discussed in previous work~\cite{olver2013fast}. Also notice that PGDA converges faster than PGITS given the same degree of polynomial. This empirical result validates our theoretical results that suggest PGDA is more efficient than PGITS.

We also report the symmetric KL divergence as a function of runtime in Figure \ref{fig:ising-time}c. Similar to the previous section, the two Poisson-Gibbs methods are faster than plain Gibbs sampling.   

\subsection{Truncated Gaussian Mixture}\label{sec:gmm}
We further demonstrate PGITS and PGDA on a truncated Gaussian mixture model. We consider the following Gaussian mixture with tied means as done in previous work \cite{welling2011bayesian,li2017mini}:
\[
 x_1 \sim \Ncal(0, \sigma_1^2), \; x_2 \sim \Ncal(0, \sigma_2^2), \;
y_i \sim \frac{1}{2} \Ncal( x_1, \sigma_y^2) + \frac{1}{2}\Ncal( x_1 +  x_2, \sigma_y^2).
\]
We used the same parameters as in \citet{welling2011bayesian}: $\sigma_1^2 = 10$, $\sigma_2^2 = 1$, $\sigma_y^2 = 2$, $ x_1 = 0$ and $ x_2 = 1$. This posterior has two modes at $(x_1, x_2) = (0,1)$ and $(x_1, x_2) = (1,-1)$. We truncate the posterior by bounding the variables $x_1$ and $x_2$ in $[-6, 6]$. The energy can be written as
\[
U(x) = \log p(x_1) + \log p(x_2) + \sum_{i=1}^N \log p(y_i|x_1, x_2)
\]
which can be regarded as a factor graph with $N$ factors. We add a positive constant to the energy to ensure each factor is non-negative: this will not change the underlying distribution.
As in \citet{li2017mini}, we set $N=10^6$. $L = 1581.14$ for this model and we set $\lambda = 500$, $m=20$ and $k=25$. We have also considered higher values of $\lambda$ and found that the results are very similar. We generate $10^6$ samples for all methods. A uniform distribution in $[-6, 6]$ is used as the proposal distribution in Gibbs with rejection sampling. We try varying values for $w$ but none of them results in reasonable density estimate which may be due to the inefficiency of rejection sampling \cite{olver2013fast}. We report the results when the average needed steps for a sample to be accepted is around 1000. The average number of factors being evaluated per iteration of Poisson-Gibbs is 1802.
Our results are reported in Figure~\ref{fig:mog}, where we observe visually that the density estimates of PGITS and PGDA are very accurate. In contrast, rejection sampling completely failed to estimate the density given the budget.

\begin{figure*}[t!]
\centering
	{\setlength{\tabcolsep}{0pt}\begin{tabular}{cccc}
	\includegraphics[width=3.5cm]{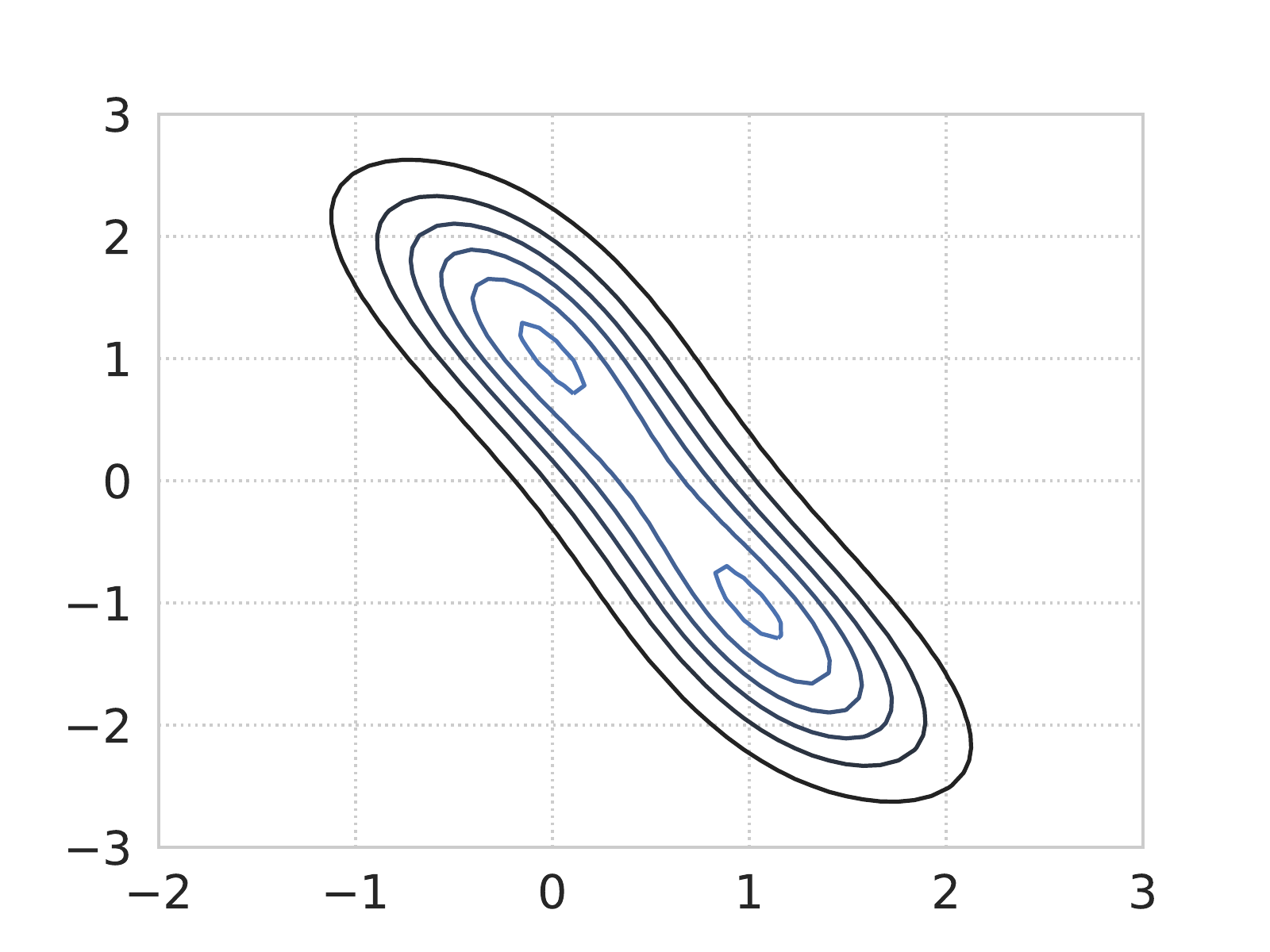}  &
        \includegraphics[width=3.5cm]{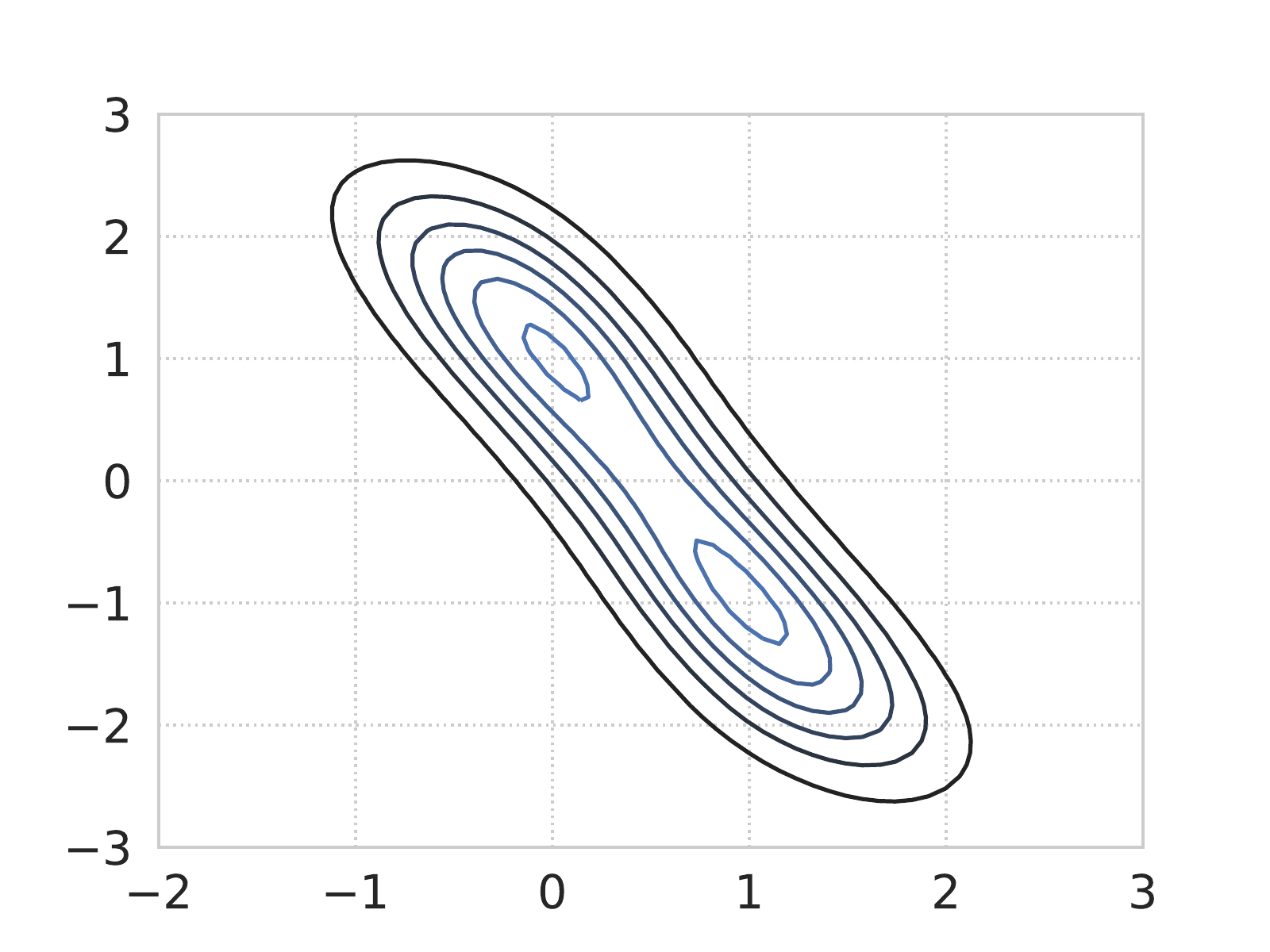}&
		\includegraphics[width=3.5cm]{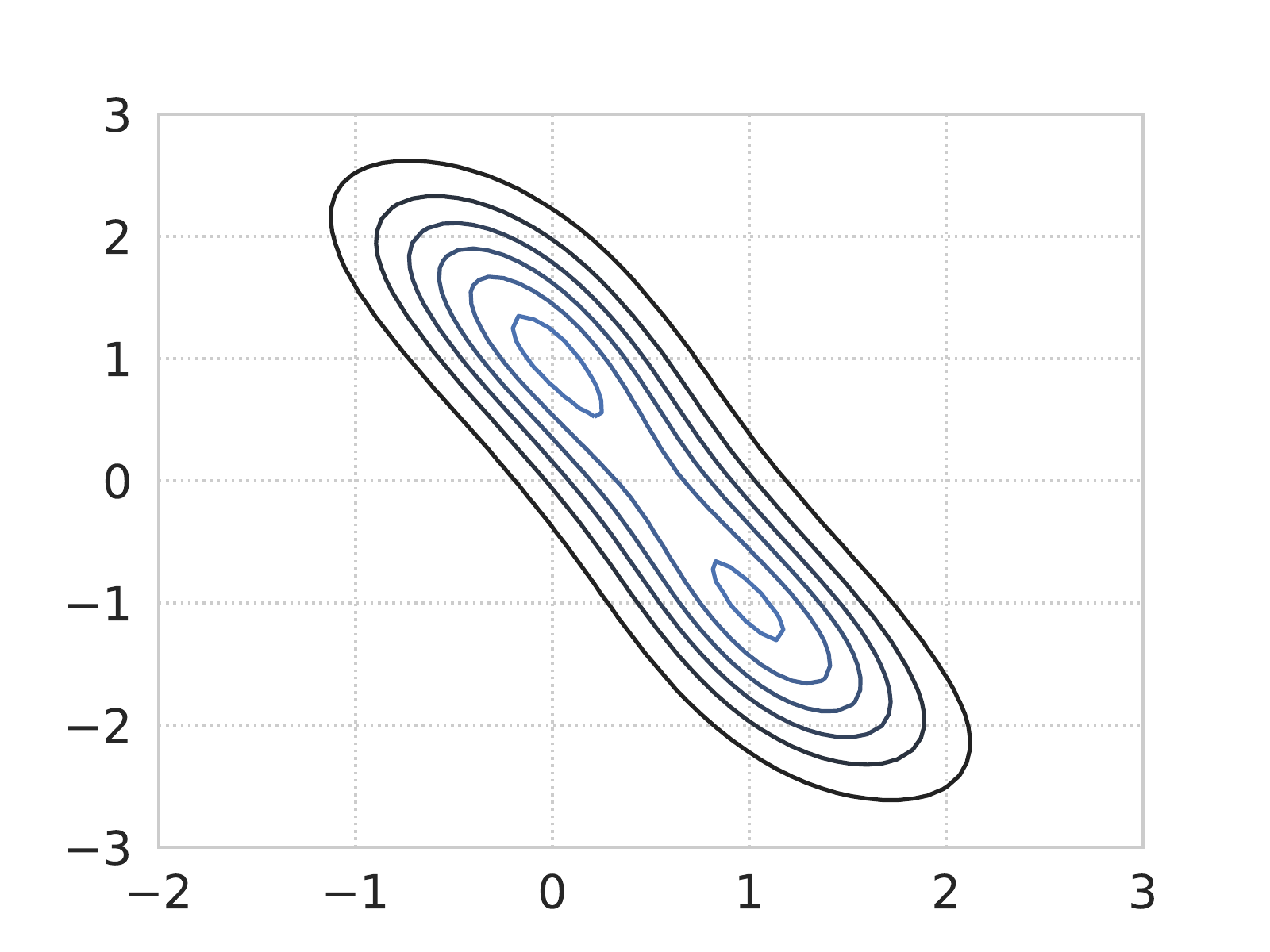}  &
		\includegraphics[width=3.5cm]{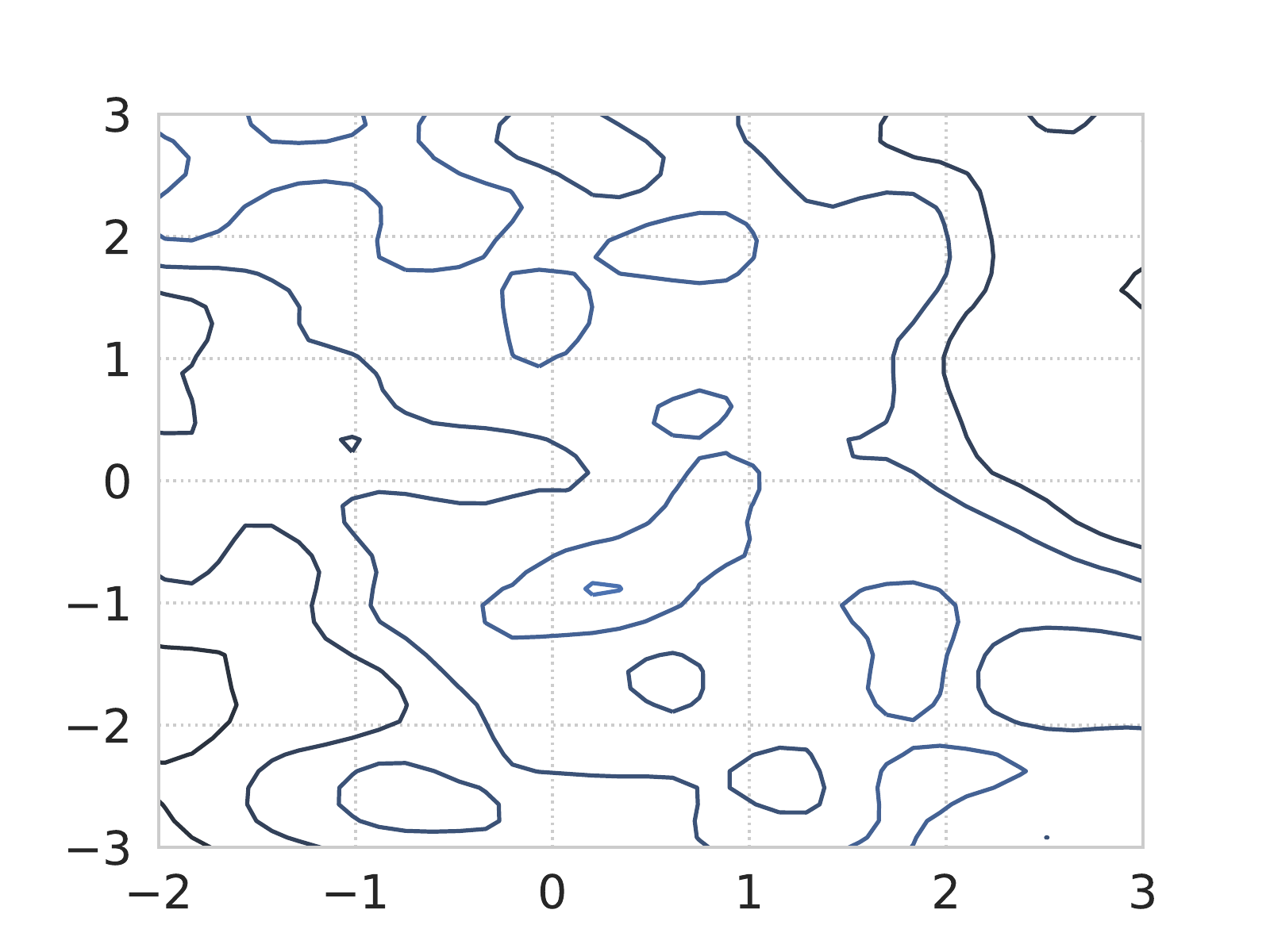}
		\\
		(a) True&
		(b) PGITS&
		(c) PGDA&
		(d) Gibbs-rejection
		\hspace{-0mm}\\		
	\end{tabular}}
	\caption{A visualization of the estimated density on a truncated Gaussian mixture model.}
	\label{fig:mog}
\end{figure*}

\section{Conclusion}
We propose Poisson-minibatching Gibbs sampling to generate unbiased samples with theoretical guarantees on the convergence rate.
Our method provably converges to the desired stationary distribution at a rate that is at most a constant factor slower than the full batch method, as measured by the spectral gap.
We provide guidance about how to set the hyperparameters of our method to make the convergence speed arbitrarily close to the full batch method.
On continuous state spaces, we propose two variants of Poisson-Gibbs based on fast inverse transform sampling and provide convergence analysis for both of them.
We hope that our work will help inspire more exploration into unbiased and guaranteed-fast stochastic MCMC methods.

\subsection*{Acknowledgements}
This work was supported by a gift from Huawei. We thank Wing Wong for the helpful discussion.

\bibliography{neurips_2019}
\bibliographystyle{plainnat}

\appendix
\newpage
\noindent\makebox[\linewidth]{\rule{\linewidth}{3.5pt}}
\begin{center}
	\bf{\Large Supplementary Material:
Poisson-Minibatching for Gibbs Sampling with Convergence Rate Guarantees}
\end{center}
\noindent\makebox[\linewidth]{\rule{\linewidth}{1pt}}

\section{Fast Sampling of the Auxiliary Variables}
In this section, we describe in detail the method used to sample the auxiliary variables $s_{\phi}$ and prove Statement~\ref{stmt:samplingphi}.
The method for doing so is described here in Algorithm~\ref{alg:samplingphi}.
\begin{algorithm}[H]
  \caption{Sample auxiliary variables $s_{\phi}$}
  \begin{algorithmic}
    \label{alg:samplingphi}
    \STATE $\triangleright$ pre-computation step; happens once
    \FOR{$i = 1$ \textbf{to} $n$}
        \STATE $\Lambda_i \leftarrow \sum_{\phi \in A[i]} \frac{\lambda M_{\phi}}{L} + M_{\phi}$
        \STATE \textbf{compute distribution} $\rho_i$ over $A[i]$ where
        \[
            \rho_i(\phi) \propto \frac{\lambda M_{\phi}}{L} + M_{\phi}.
        \]
        \STATE \textbf{process distribution} $\rho_i$ so that in future, it can be sampled from in constant time
    \ENDFOR
    \STATE
    \STATE $\triangleright$ to actually re-sample the auxiliary variables 
    \STATE \textbf{given:} current state $x \in \Omega$, variable $i$ to resample
    \STATE \textbf{initialize sparse vector} $s: A[i] \rightarrow \Z$
    \STATE \textbf{sample} $B \sim \operatorname{Poisson}(\Lambda_i)$
    \FOR{$b = 1$ \textbf{to} $B$}
        \STATE \textbf{sample} $\phi \sim \rho_i$
        \STATE \textbf{compute} $\phi(x)$
        \STATE \textbf{with probability} $\frac{ \frac{\lambda M_{\phi}}{L} + \phi(x) }{ \frac{\lambda M_{\phi}}{L} + M_{\phi} }$ \textbf{update} sparse vector $s_{\phi} \leftarrow s_{\phi} + 1$
    \ENDFOR
  \end{algorithmic}
\end{algorithm}

To see that this is valid, let $B = \sum_i^n s_i$ where $s_i$ are Poisson variables with parameters $\lambda_i$.
We know that $B$ is also Poisson distributed with parameter $\Lambda = \sum_i^n \lambda_i$.
Conditioned on the value of $B$, it is known that $s_i$ follows a multinomial distribution with event probabilities $\lambda_i/\Lambda$ and trial count $B$. Therefore, we can first sample $B \sim \text{Poisson}(\Lambda)$ and then sample
\[
    (s_1, \dots s_n) \sim \operatorname{Multinomial}\left(B, \left(\frac{\lambda_1}{\Lambda}, \ldots, \frac{\lambda_n}{\Lambda} \right) \right).
\]
Our Algorithm~\ref{alg:samplingphi} is only slightly more complicated than this process, in order to minimize the number of times that $\phi(x)$ is evaluated, but it can be seen to produce the valid distribution by the same reasoning.

The computational cost of Algorithm~\ref{alg:samplingphi} is clearly proportional to $B$, and since
\[
    \mathbf{E}[B] = \Lambda_i = \sum_{\phi \in A[i]} \frac{\lambda M_{\phi}}{L} + M_{\phi} \le \lambda + L,
\]
it follows that the overall average computational cost will also be $\lambda + L$.
This proves Statement~\ref{stmt:samplingphi}.

\section{Poisson-Gibbs with Exact Sampling from the Conditional Distribution}
\subsection{Derivation of the joint distribution}\label{app:joint-dist}
In this subsection, we derive the joint distribution (\ref{eq:pi}) by substituting the distributions of $x$ and $s$ into the conditional distribution of $s$ given $x$. By the expression of Poisson distribution for $s_{\phi}$ and the independence of $s_{\phi}$, we have
\begin{align*}
  \pi(x, s) 
  &= 
  \pi(x)\pi(s|x)\\
  &\propto
  \exp\left(\sum_{\phi\in\Phi} \phi(x)\right)\prod_{\phi\in\Phi}\pi(s_{\phi}|x)\\
  &=
  \exp\left(\sum_{\phi\in\Phi}\left( \phi(x)
  +
  \log\pi(s_{\phi}|x)\right)\right)\\
  &=
  \exp\left(\sum_{\phi\in\Phi}\left(\phi(x)
  +
  s_{\phi}\log\left(\frac{\lambda M_{\phi}}{L}+\phi(x)\right)
  -\log (s_{\phi}!)
  -\frac{\lambda M_{\phi}}{L}-\phi(x)\right)\right)\\
  &=
  \exp\left(\sum_{\phi\in\Phi}\left(
  s_{\phi}\log\left(\frac{\lambda M_{\phi}}{L}
  +
  \phi(x)\right)
  -\log (s_{\phi}!)-\frac{\lambda M_{\phi}}{L}\right)\right)\\
  &\propto
  \exp\left(\sum_{\phi\in\Phi}\left(
  s_{\phi}\log\left(\frac{\lambda M_{\phi}}{L}
  +
  \phi(x)\right)
  -\log (s_{\phi}!)\right)\right)\\
  &=
  \exp\left(
  \sum_{\phi\in\Phi} \left(s_{\phi} \log\left( 1 + \frac{L}{\lambda M_{\phi}}\phi(x) \right) + s_{\phi}\log\left(\frac{\lambda M_{\phi}}{L}\right) - \log\left( s_{\phi}! \right)\right)\right).
\end{align*}

\subsection{Proof of Theorem \ref{thm:discrete}}

In this section, we prove that Poisson-Gibbs converges, and derive a bound on its convergence rate.

\begin{proof}
First, we will derive an expression for the transition operator of Poisson-Gibbs chain, and show it is reversible.
Then we will bound the spectral gap.

If $x$ and $y$ are states which differ in only one variable $i$, the probability of transitioning from $x$ to $y$ will be the probability of choosing to sample variable $i$ times the expected value over the random choice of $s$ of the probability of sampling $y(i)$ from $\rho$.
That is,
\begin{align*}
  T(x, y)
  &= 
  \frac{1}{n}
  \cdot
  \Exv{
    \rho(y(i))
  }\\
  &=
  \frac{1}{n}
  \cdot
  \Exv{
    \frac{
      \exp(U_{y(i)})
    }{
      \int \exp(U_u) \; du
    }
  }\\
  &=
  \frac{1}{n}
  \cdot
  \sum_s
  \frac{
    \exp(U_{y(i)})
  }{
    \int \exp(U_u) \; du
  }
  \cdot
  \prod_{\phi \in A[i]}
  \frac{1}{s_{\phi}!}
  \left( \frac{\lambda M_{\phi}}{L} + \phi(x) \right)^{s_{\phi}}
  \exp\left( - \left( \frac{\lambda M_{\phi}}{L} + \phi(x) \right) \right)\\
  &=
  \frac{1}{n}
  \cdot
  \sum_s
  \frac{
    \exp\left( \sum_{\phi \in A[i]} s_{\phi} \log\left( \frac{\lambda M_{\phi}}{L} + \phi(y) \right) \right)
  }{
    \int
    \exp\left( \sum_{\phi \in A[i]} s_{\phi} \log\left( \frac{\lambda M_{\phi}}{L} + \phi(z_u) \right) \right) \; du
  }
  \\&\hspace{2em}\cdot
  \prod_{\phi \in A[i]}
  \left( 1 + \frac{L}{\lambda M_{\phi}} \phi(x) \right)^{s_{\phi}}
  \cdot\exp\left( - \phi(x)  \right)
  \\&\hspace{2em}\cdot
  \prod_{\phi \in A[i]}
  \frac{1}{s_{\phi}!}
  \left( \frac{\lambda M_{\phi}}{L}  \right)^{s_{\phi}}
  \cdot\exp\left( - \frac{\lambda M_{\phi}}{L}  \right)
\end{align*}
where $z_u$ denotes $x$ where $x(i)$ has been set equal to $u$.
Note that $s_{\phi}$ here are non-negative integers that a Poisson variable can take, not variables. So if we let $r_{\phi} \sim \text{Poisson}\left( \frac{\lambda M_{\phi}}{L} \right)$ and $r_{\phi}$ to be all independent, we can write this as
\begin{align*}
  T(x, y)
  &=
  \frac{1}{n}
  \cdot
  \textbf{E}_r \Bigg[
  \frac{
    \exp\left( \sum_{\phi \in A[i]} r_{\phi} \log\left( \frac{\lambda M_{\phi}}{L} + \phi(y) \right) \right)
  }{
    \int
    \exp\left( \sum_{\phi \in A[i]} r_{\phi} \log\left( \frac{\lambda M_{\phi}}{L} + \phi(z_u) \right) \right) \; du
  }
  \\&\hspace{2em}\cdot
  \prod_{\phi \in A[i]}
  \left( 1 + \frac{L}{\lambda M_{\phi}} \phi(x) \right)^{r_{\phi}}
  \cdot\exp\left( - \phi(x)  \right) \Bigg] \\
  &=
  \frac{1}{n}
  \cdot
  \textbf{E}_r \Bigg[
  \frac{
    \exp\left( \sum_{\phi \in A[i]} r_{\phi} \left( 
        \log\left( 1 + \frac{L}{\lambda M_{\phi}} \phi(y) \right) 
        + 
        \log\left( 1 + \frac{L}{\lambda M_{\phi}} \phi(x) \right)
    \right) \right)
  }{
    \int
    \exp\left( \sum_{\phi \in A[i]} r_{\phi} \log\left( 1 + \frac{L}{\lambda M_{\phi}} \phi(z_u) \right) \right) \; du
  }
  \\&\hspace{2em}
  \cdot\exp\left( - \sum_{\phi \in A[i]} \phi(x)  \right) \Bigg] \\
\end{align*}
Therefore, since
\[
  \pi(x) = \frac{1}{Z} \cdot \exp\left( \sum_{\phi\in\Phi} \phi(x) \right),
\]
it follows that
\begin{align*}
  &\hspace{-1em}\pi(x) T(x, y)\\
  &=
  \frac{1}{n Z}
  \cdot
  \textbf{E}_r \Bigg[
  \frac{
    \exp\left( \sum_{\phi \in A[i]} r_{\phi} \left( 
        \log\left( 1 + \frac{L}{\lambda M_{\phi}} \phi(y) \right) 
        + 
        \log\left( 1 + \frac{L}{\lambda M_{\phi}} \phi(x) \right)
    \right) \right)
  }{
    \int
    \exp\left( \sum_{\phi \in A[i]} r_{\phi} \log\left( 1 + \frac{L}{\lambda M_{\phi}} \phi(z_u) \right) \right) \; du
  }
  \\&\hspace{2em}
  \cdot\exp\left( \sum_{\phi \in \Phi} \phi(x) - \sum_{\phi \in A[i]} \phi(x)  \right) \Bigg] \\
  &=
  \frac{\exp(U_{\neg i}(x))}{n Z}
  \cdot
  \textbf{E}_r\left[
  \frac{
    \exp\left( \sum_{\phi \in A[i]} r_{\phi} \left( 
        \log\left( 1 + \frac{L}{\lambda M_{\phi}} \phi(y) \right) 
        + 
        \log\left( 1 + \frac{L}{\lambda M_{\phi}} \phi(x) \right)
    \right) \right)
  }{
    \int
    \exp\left( \sum_{\phi \in A[i]} r_{\phi} \log\left( 1 + \frac{L}{\lambda M_{\phi}} \phi(z_u) \right) \right) \; du
  } \right].
\end{align*}
where we define $U_{\neg i}(x) = \sum_{\phi\notin A[i]}\phi(x)$. This expression is symmetric in $x$ and $y$ (note that $U_{\neg i}(x)$ does not depend on variable $i$), so it follows that the Markov chain is reversible, and its stationary distribution is indeed $\pi$.

We can proceed to try to bound its spectral gap, using the technique of Dirichlet forms. We start by simplifying our expression by defining 
\[
  \bar \phi(x) = \frac{ L \phi(x) }{ \lambda M_{\phi} }.
\]
Using this, we get
\begin{align*}
  \pi(x) T(x, y)
  &=
  \frac{\exp(U_{\neg i}(x))}{n Z}
  \cdot
  \textbf{E}_r\left[
  \frac{
    \exp\left( \sum_{\phi \in A[i]} r_{\phi} \left( 
        \log\left( 1 + \bar \phi(y) \right) 
        + 
        \log\left( 1 + \bar \phi(x) \right)
    \right) \right)
  }{
    \int
    \exp\left( \sum_{\phi \in A[i]} r_{\phi} \log\left( 1 + \bar \phi(z_u) \right) \right) \; du
  } \right].
\end{align*}
We proceed by bringing the exponential on the top of this sum down to the bottom and inside the integral, which produces
\begin{align*}
  \pi(x)T(x, y)
  &=
  \frac{\exp\left(U_{\neg i}(x)\right)}{n Z}
  \cdot
  \mathbf{E}_r\Bigg[
    \Bigg(
      \int
      \exp\Bigg( 
        \sum_{\phi \in A[i]} r_{\phi} \Bigg(
          \log\left( 1 + \bar \phi(z_u) \right) 
          \\&\hspace{2em} -
          \log\left( 1 + \bar \phi(x) \right)
          -
          \log\left( 1 + \bar \phi(y) \right)
        \Bigg)
      \Bigg) \; du
    \Bigg)^{-1}
  \Bigg]\\
  &\ge
  \frac{\exp\left(U_{\neg i}(x)\right)}{n Z}
  \cdot
  \Bigg(\mathbf{E}_r\Bigg[
    \int
    \exp\Bigg( 
      \sum_{\phi \in A[i]} r_{\phi} \Bigg(
          \log\left( 1 + \bar \phi(z_u) \right) 
          \\&\hspace{2em} -
          \log\left( 1 + \bar \phi(x) \right)
          -
          \log\left( 1 + \bar \phi(y) \right)
      \Bigg)
    \Bigg) \; du
  \Bigg]\Bigg)^{-1}
\end{align*}
where this inequality follows from Jensen's inequality and the fact that $1/x$ is convex.
By converting the exp-of-sum to a product-of-exp, and recalling that the $r_{\phi}$ are independent, we can further reduce this to
\begin{align*}
  \pi(x) T(x, y)
  &\ge
  \frac{\exp\left(U_{\neg i}(x)\right)}{n Z}
  \Bigg(
    \int
    \mathbf{E}_r\Bigg[
      \prod_{\phi \in A[i]}
      \exp\Bigg( 
        r_{\phi} \Bigg(
          \log\left( 1 + \bar \phi(z_u) \right) 
          \\&\hspace{2em} -
          \log\left( 1 + \bar \phi(x) \right)
          -
          \log\left( 1 + \bar \phi(y) \right)
        \Bigg)
      \Bigg)
    \Bigg]du
  \Bigg)^{-1}\\
  &=
  \frac{\exp\left(U_{\neg i}(x)\right)}{n Z}
  \Bigg(
    \int
    \prod_{\phi \in A[i]}
    \mathbf{E}_r\Bigg[
      \exp\Bigg( 
        r_{\phi} \Bigg(
          \log\left( 1 + \bar \phi(z_u) \right) 
          \\&\hspace{2em} -
          \log\left( 1 + \bar \phi(x) \right)
          -
          \log\left( 1 + \bar \phi(y) \right)
        \Bigg)
      \Bigg)
    \Bigg]du
  \Bigg)^{-1}.
\end{align*}
This final expectation expression is just the moment generating function of the Poisson random variable $r_{\phi}$ evaluated at
\[
  t
  =
  \log\left( 1 + \bar \phi(z_u) \right) 
  -
  \log\left( 1 + \bar \phi(x) \right)
  -
  \log\left( 1 + \bar \phi(y) \right).
\]
Here, from the standard formula for that MGF, we get
\[
  \mathbf{E}_r [\exp(r_{\phi} t)]
  =
  \exp\left(
    \frac{\lambda M_{\phi}}{L} \left(
      \exp(t) - 1
    \right)
  \right)
\]
So
\begin{align*}
  &\exp(t) - 1\\
  &=
  \frac{
    1 + \bar \phi(z_u)
  }{
    (1 + \bar \phi(x))
    (1 + \bar \phi(y))
  }
  -
  1\\
  &=
  \frac{
    \bar \phi(z_u)
    -
    \bar \phi(x)
    -
    \bar \phi(y)
    -
    \bar \phi(x)
    \bar \phi(y)
  }{
    (1 + \bar \phi(x))
    (1 + \bar \phi(y))
  }\\
  &=
  \bar \phi(z_u)
  -
  \bar \phi(x)
  -
  \bar \phi(y)
  -
  \frac{
    \left(
      \bar \phi(z_u)
      -
      \bar \phi(x)
      -
      \bar \phi(y)
    \right)
    \left(
      \bar \phi(x)
      +
      \bar \phi(y)
      +
      \bar \phi(x)
      \bar \phi(y)
    \right)
    +
    \bar \phi(x)
    \bar \phi(y)
  }{
    (1 + \bar \phi(x))
    (1 + \bar \phi(y))
  }.
\end{align*}
Since
\[
  0 \le \bar \phi(x) = \frac{ L \phi(x) }{ \lambda M_{\phi} } \le \frac{L}{\lambda} \le \frac{1}{2}
\]
(where here we're using the condition in the theorem statement that $2 L \le \lambda$) we can bound this with
\begin{align*}
  &\exp(t) - 1\\
  &\le
  \bar \phi(z_u)
  -
  \bar \phi(x)
  -
  \bar \phi(y)
  -
  \frac{
    \left(
      -
      \bar \phi(x)
      -
      \bar \phi(y)
    \right)
    \left(
      \bar \phi(x)
      +
      \bar \phi(y)
    \right)
    +
    \left(
      1
      -
      \bar \phi(x)
      -
      \bar \phi(y)
    \right)
    \bar \phi(x)
    \bar \phi(y)
  }{
    (1 + \bar \phi(x))
    (1 + \bar \phi(y))
  }\\
  &\le
  \bar \phi(z_u)
  -
  \bar \phi(x)
  -
  \bar \phi(y)
  +
  \frac{
    \left(
      \bar \phi(x)
      +
      \bar \phi(y)
    \right)
    \left(
      \bar \phi(x)
      +
      \bar \phi(y)
    \right)
  }{
    (1 + \bar \phi(x))
    (1 + \bar \phi(y))
  }\\
  &\le
  \bar \phi(z_u)
  -
  \bar \phi(x)
  -
  \bar \phi(y)
  +
  \left(
    \bar \phi(x)
    +
    \bar \phi(y)
  \right)^2\\
  &\le
  \bar \phi(z_u)
  -
  \bar \phi(x)
  -
  \bar \phi(y)
  +
  \frac{4L^2}{\lambda^2}.
\end{align*}
So,
\begin{align*}
  \mathbf{E}{\exp(r_{\phi} t)}
  &=
  \exp\left(
    \frac{\lambda M_{\phi}}{L} \left(
      \exp(t) - 1
    \right)
  \right)\\
  &\le
  \exp\left(
    \frac{\lambda M_{\phi}}{L} \left(
      \bar \phi(z_u)
      -
      \bar \phi(x)
      -
      \bar \phi(y)
      +
      \frac{4L^2}{\lambda^2}
    \right)
  \right)\\
  &=
  \exp\left(
    \phi(z_u)
    -
    \phi(x)
    -
    \phi(y)
    +
    \frac{4 L M_{\phi}}{\lambda}
  \right).
\end{align*}
Substituting this into the original expression produces
\begin{align*}
  &\pi(x) T(x, y)\\
  &\ge
  \frac{\exp\left(U_{\neg i}(x)\right)}{n Z}
  \left(
    \int
    \prod_{\phi \in A[i]}
    \exp\left(
      \phi(z_u)
      -
      \phi(x)
      -
      \phi(y)
      +
      \frac{4L M_{\phi}}{\lambda}
    \right)du
  \right)^{-1}\\
  &=
  \frac{\exp\left(U_{\neg i}(x)\right)}{n Z}
  \left(
    \int
    \exp\left(
      \sum_{\phi \in A[i]}
      \phi(z_u)
      -
      \sum_{\phi \in A[i]}
      \phi(x)
      -
      \sum_{\phi \in A[i]}
      \phi(y)
      +
      \sum_{\phi \in A[i]}
      \frac{4L M_{\phi}}{\lambda}
    \right)du
  \right)^{-1}\\
  &\ge
  \frac{\exp\left(U_{\neg i}(x)\right)}{n Z}
  \left(
    \int
    \exp\left(
      \sum_{\phi \in A[i]}
      \phi(z_u)
      -
      \sum_{\phi \in A[i]}
      \phi(x)
      -
      \sum_{\phi \in A[i]}
      \phi(y)
      +
      \frac{4L^2}{\lambda}
    \right)du
  \right)^{-1}\\
  &=
  \exp\left( - \frac{4L^2}{\lambda} \right)
  \frac{\exp\left(U_{\neg i}(x)\right)}{n Z}
  \left(
    \int
    \exp\left(
      \bar U_u
      -
      \bar U_{x(i)}
      -
      \bar U_{y(i)}
    \right)du
  \right)^{-1}\\
  &=
  \exp\left( - \frac{4L^2}{\lambda} \right)
  \frac{\exp\left(U_{\neg i}(x)\right)}{n Z}
  \frac{
    \exp(\bar U_{x(i)}) \cdot \exp(\bar U_{y(i)})
  }{
    \int \exp( \bar U_u )du
  }\\
  &=
  \exp\left( - \frac{4L^2}{\lambda} \right)
  \frac{1}{n Z}
  \frac{
    \exp(U(x)) \cdot \exp(\bar U_{y(i)})
  }{
    \int \exp( \bar U_u )du
  }
\end{align*}
where $\bar U_v$ denotes the assignment of $U_v$ in the plain Gibbs sampling algorithm (Algorithm~\ref{alg:gibbs}),
\[
    \bar U_v = \sum_{\phi \in A[i]} \phi(z_v),
\]
Finally, if we let $G$ denote the transition probability operator of plain Gibbs sampling, we notice right away that
\begin{align*}
  \pi(x) T(x, y)
  &\ge
  \exp\left( - \frac{4L^2}{\lambda} \right)
  \frac{1}{n Z}
  \frac{
    \exp(U(x)) \cdot \exp(\bar U_{y(i)})
  }{
    \int \exp( \bar U_u )du
  }\\
  &=
  \exp\left( - \frac{4L^2}{\lambda} \right)
  \pi(x) G(x, y).
\end{align*}

We will use the Dirichlet form argument to finish the proof. A real function $f$ is square integrable with respect to probability measure $\pi$, if it satisfies
\[
\int f(x)^2\pi(dx)<\infty.
\]
Define $L^2(\pi)$ to be the Hilbert space of all such functions.

Let $L^2_0(\pi)\subset L^2(\pi)$ to be the Hilbert space that uses the same inner product but only contains functions such that
\[
\mathbf{E}_{\pi}[f] = \int f(x)\pi(dx) = 0.
\]

We also define the notation 
\begin{align*}
\langle f,g \rangle = \int f(x)g(x)\pi(dx).
\end{align*}
A special example is $\operatorname{Var}_{\pi}[f] = \langle f,f\rangle$. 

From here, the Dirichlet form of a Markov chain associated with transition operator $T$ is given by \cite{fukushima2010dirichlet}
\begin{align*}
\mathbf{E}(f) 
	= 
	\frac{1}{2}\int\int\left(f(x)-f(y)\right)^2 T(x,y)\pi(x)dxdy.
\end{align*}

And the spectral gap can be written as \cite{aida1998uniform}
\[
\gamma = \inf_{f\in L^2_0(\pi): \operatorname{Var}_{\pi}[f] = 1} \mathbf{E}(f).
\]
The spectral gap is related to other common measurement of the convergence of MCMC. For example, it has the following relationship with the mean squared error $e_{\pi}$ on a Markov chain $\{X_n\}_{n\in \mathbb{N}}$ \cite{rudolf2011explicit},
\[
e_{\pi}^2 \leq \frac{2}{n\gamma}\norm{f}_2^2.
\]

With the expression of the spectral gap, it follows that
\begin{align*}
\bar{\gamma} &= \inf_{f\in L^2_0(\pi): Var_{\pi}[f] = 1} \left[\frac{1}{2}\int\int\left(f(x)-f(y)\right)^2T(x,y)\pi(x) \; dx \; dy\right]\\
&\geq \exp\left( - \frac{4L^2}{\lambda} \right)\cdot \inf_{f\in L^2_0(\pi): Var_{\pi}[f] = 1} \left[\frac{1}{2}\int\int\left(f(x)-f(y)\right)^2G(x,y)\pi(x) \; dx \; dy\right]\\
&= \exp\left( - \frac{4L^2}{\lambda} \right) \cdot \gamma.
\end{align*}

This proves the theorem.
\end{proof}

\section{Poisson-Gibbs on Continuous State Spaces}
\subsection{Poisson-Gibbs with Fast Inverse Transform Sampling (PGITS)} 
In the main body of the paper, we mentioned the PGITS method, Poisson-Gibbs with Fast Inverse Transform Sampling.
This method is to approximate the PDF by Chebyshev polynomials and then use inverse transform sampling.
In this section, we will outline the algorithm and derive convergence rate results for it.
These results will illustrate why PGITS can be expected to perform worse than PGDA.

PGITS operates by approximating the PDF with a Chebyshev polynomial approximation and then sampling from that polynomial approximation using inverse transform sampling.
Specifically, if the PDF we want to sample from is $f(x)$, we can approximate $f$ by $\tilde{f}$ on $[a, b]$ using Chebyshev polynomials,
\begin{align}\label{eq:approx}
\tilde{f} = \sum_{k=0}^m\alpha_k T_k\left(\frac{2(x - a)}{b-a} - 1\right),\ \alpha_k\in \RR,\ x\in[a, b]
\end{align}
where $T_k (x) = \cos(k \cos^{-1} x)$ is the degree $k$ Chebyshev polynomial, and $\alpha_k$ are the Chebyshev coefficients of the function $f$~\cite{trefethen2013approximation}.
We do this by interpolating $f$ at its Chebyshev nodes, resulting in $\tilde f$ being the $m$th order \emph{Chebyshev interpolant}.
Once we have the polynomial approximation $\tilde{f}$ we can construct the corresponding CDF approximation $\tilde{F}$ by calculating the integral directly (since polynomials are straightforward to integrate).
With the approximation $\tilde{F}$, we are able to use inverse transform sampling to generate samples. We call this whole algorithm {\it PGITS} and it is listed as Algorithm~\ref{alg:PGITS}.

We show that PGITS is reversible and bound its spectral gap in the following theorem.
\begin{theorem}\label{thm:poly}
  PGITS (Algorithm~\ref{alg:PGITS}) is reversible and has a stationary distribution $\pi$.
  Let $\bar{\gamma}$ denote its spectral gap, and let $\gamma$ denote the spectral gap of plain Gibbs sampling. Assume $\rho > 1$ is some constant such that every factor function $\phi$, treated as a function of any single variable $x_i$, must
   be analytically continuable to the Bernstein ellipse with radius parameter $\rho$ shifted-and-scaled so that its foci are at $a$ and $b$, such that it satisfies $| \phi(z) | \le M_{\phi}$ anywhere in that ellipse. Then, if $\lambda \ge 2 L$ it will hold that
  \[
    \bar \gamma
    \ge
    \left(1 - \frac{8 \exp(L) \rho^{-m/2}}{\sqrt{\rho - 1}} \right)
    \cdot
    \exp\left(-\frac{4L^2}{\lambda}\right) \cdot \gamma.
  \]
\end{theorem}
We can set $m = \Theta(L)$ and $\lambda = \Theta(L^2)$ to make the ratio of the spectral gaps $O(1)$, which is independent of the size of the problem. If the parameters are set in this way, the total cost of PGITS is $O(m\cdot (\lambda + L)) = O(L\cdot L^2) = O(L^3)$. 

\begin{algorithm}[t]
  \caption{PGITS: Poisson-Gibbs Inverse Transform Sampling}
  \begin{algorithmic}
    \label{alg:PGITS}
    \STATE \textbf{given:} state $x \in \Omega$, degree $m$, domain $[a,b]$
    \LOOP
      \STATE \textbf{set} $i$, $s_{\phi}$, $S$, and $U$ as in Algorithm~\ref{alg:poisson-gibbs}.
      \STATE \textbf{construct} degree-$m$ Chebyshev polynomial approximation of polynomial PDF on $[a,b]$
      \[ \tilde{f}(v) \approx \exp(U_v) \]
      \STATE \textbf{compute} the CDF polynomial
      \[
        \tilde{F}(v) = \left( \int_a^b \tilde f(y) \; dy \right)^{-1} \int_a^v \tilde f(y) \; dy 
      \]
      \STATE \textbf{sample} $u \sim \operatorname{Unif}[0,1]$.
      \STATE \textbf{solve} root-finding problem for $v$: $\tilde{F}(v) = u$
       \STATE $\triangleright$ Metropolis-Hastings correction:
       \STATE \[ p \leftarrow \frac{\exp(U_{v}) \cdot \tilde{f}(x(i))}{\exp(U_{x(i)}) \cdot \tilde{f}(v)} \]
       \STATE \textbf{with probability} $\min(1,p)$, set $x(i) \leftarrow v$
      \STATE \textbf{output sample} $x$
    \ENDLOOP
  \end{algorithmic}
\end{algorithm}

\subsubsection{Proof of Theorem \ref{thm:poly}}\label{sec:proof-thm-poly}

\begin{proof}
Similar to the previous analysis of Poisson-Gibbs, we will show the PGITS is reversible by using the expression of the transition operator.
Then we will bound the spectral gap.

Let $T_{i,s}(x, y)$ denote the probability of transitioning from state $x$ to $y$ given that we have already chosen to sample variable $i$ with minibatch coefficients $s$. Then, the overall transition operator will be 
\[
T (x, y) = \Exv{T_{i,s}(x, y)}
\]
where the expectation is taken over $i$ and $s$.

Let the polynomial interpolant for $\exp(U_v)$ be $\tilde{f}(v)$ which is given in (\ref{eq:approx}).
Note that this interpolant is a function of the index $i$ and the minibatch coefficients $s$.
Then,
\begingroup
\allowdisplaybreaks
\begin{align*}
  T_{i,s}(x, y)
  &= 
    \rho(y(i))
    \cdot
    \min (1,a)\\
  &=
    \frac{
      \tilde{f}(y(i))
    }{
      \int \tilde{f}(u)du
    }
    \cdot
    \min\left(1, \frac{\exp(U_{y(i)})\tilde{f}(x(i))}{\exp(U_{x(i)})\tilde{f}(y(i))}\right)
\end{align*}  
Therefore,
\begin{align*}
  T(x, y)
  &=
  \frac{1}{n}\mathbf{E}{
    \frac{
      \tilde{f}(y(i))
    }{
      \int \tilde{f}(u)du
    }
    \cdot
    \min\left(1, \frac{\exp(U_{y(i)})\tilde{f}(x(i))}{\exp(U_{x(i)})\tilde{f}(y(i))}\right)}\\
  &=
  \frac{1}{n}\mathbf{E}{
    \frac{
      1
    }{
      \int \tilde{f}(u)du
    }
    \cdot
    \min\left(\tilde{f}(y(i)), \exp(U_{y(i)} - U_{x(i)})\tilde{f}(x(i))\right)}\\
     &=
  \frac{1}{n}\mathbf{E}{
    \frac{
      1
    }{
      \int \tilde{f}(u)du
    }
    \cdot
    \min\left(\tilde{f}(y(i)), \exp\left(\sum_{\phi\in A[i]}s_{\phi}\log \frac{1+\frac{L}{\lambda M_{\phi}}\phi(y)}{1+\frac{L}{\lambda M_{\phi}}\phi(x)}\right)\tilde{f}(x(i))\right)}\\
    &=
  \frac{1}{n}\sum_s
    \frac{
      1
    }{
      \int \tilde{f}(u)du
    }
    \cdot
    \min\left(\tilde{f}(y(i)), \exp\left(\sum_{\phi\in A[i]}s_{\phi}\log \frac{1+\frac{L}{\lambda M_{\phi}}\phi(y)}{1+\frac{L}{\lambda M_{\phi}}\phi(x)}\right)\tilde{f}(x(i))\right)
    \\&\hspace{2em}\cdot
    \exp\left(
    \sum_{\phi\in A[i]}
    s_{\phi} \log\left( \frac{\lambda M_{\phi}}{L} + \phi(x) \right)
    -
    \log\left( s_{\phi}! \right)
    -
    \left( \frac{\lambda M_{\phi}}{L} + \phi(x) \right)
  \right)\\
    &=
  \frac{1}{n}\sum_s
    \frac{
      1
    }{
      \int \tilde{f}(u)du
    }
    \cdot
    \min\Bigg(\tilde{f}(y(i))\exp\left(
    \sum_{\phi\in A[i]}
    s_{\phi} \log\left( 1 + \frac{L}{\lambda M_{\phi}} \phi(x) \right)\right), 
    \\&\hspace{2em}\exp\left(\sum_{\phi\in A[i]}s_{\phi}\log\left( 1+\frac{L}{\lambda M_{\phi}}\phi(y)\right)\right)\tilde{f}(x(i))\Bigg)
    \\&\hspace{2em}\cdot
    \exp\left(
    \sum_{\phi\in A[i]}\Bigg(
    s_{\phi} \log\left( \frac{\lambda M_{\phi}}{L} \right)
    -
    \log\left( s_{\phi}! \right)
    -
    \left( \frac{\lambda M_{\phi}}{L} + \phi(x) \right)
  \right)\Bigg)\\
    &=
  \frac{1}{n}\sum_s
    \frac{
      1
    }{
      \int \tilde{f}(u)du
    }
    \cdot
    \min\Bigg(\tilde{f}(y(i))\exp\left(
    \sum_{\phi\in A[i]}
    s_{\phi} \log\left( 1 + \frac{L}{\lambda M_{\phi}} \phi(x) \right)\right), 
    \\&\hspace{2em}\exp\left(\sum_{\phi\in A[i]}s_{\phi}\log\left( 1+\frac{L}{\lambda M_{\phi}}\phi(y)\right)\right)\tilde{f}(x(i))\Bigg)
    \\&\hspace{2em}\cdot
    \exp\left(
    \sum_{\phi\in A[i]}\Bigg(
    s_{\phi} \log\left( \frac{\lambda M_{\phi}}{L} \right)
    -
    \log\left( s_{\phi}! \right)
    -
    \left( \frac{\lambda M_{\phi}}{L} \right)\right)\Bigg)\exp(-U_{x(i)})
\end{align*}
\endgroup
Multiplying $\pi(x)$ on both sides,
\begin{align*}
  &\hspace{1em}\pi(x)T(x, y)\\
&=
  \frac{\exp(U_{\neg i}(x))}{nZ}\sum_s
    \frac{1}{\int \tilde{f}(u)du}\cdot
    \min\Bigg(\tilde{f}(y(i))\exp\left(
    \sum_{\phi\in A[i]}
    s_{\phi} \log\left( 1 + \frac{L}{\lambda M_{\phi}} \phi(x) \right)\right), 
    \\&\hspace{2em}\tilde{f}(x(i)) \exp\left(\sum_{\phi\in A[i]}s_{\phi}\log\left( 1+\frac{L}{\lambda M_{\phi}}\phi(y)\right)\right)\Bigg)
    \\&\hspace{2em}\cdot
    \exp\left(
    \sum_{\phi\in A[i]}\Bigg(
    s_{\phi} \log\left( \frac{\lambda M_{\phi}}{L} \right)
    -
    \log\left( s_{\phi}! \right)
    -
    \left( \frac{\lambda M_{\phi}}{L} \right)\right)\Bigg)
\end{align*}
This expression is symmetric in $x$ and $y$, so it follows that
\begin{align*}
  \pi(x) T(x, y)
  =
  \pi(y) T(y, x)
\end{align*}
Thus the Markov chain is reversible, and its stationary distribution is $\pi$.

We now bound its spectral gap, using the technique of Dirichlet forms. First, as before, we start by re-writing the chain in terms of an expectation of a new random variable $r_{\phi}$ where $r_{\phi} \sim \text{Poisson}\left( \frac{\lambda M_{\phi}}{L} \right)$ and the $r_{\phi}$ are all independent.
We also define $\bar \phi(x) = \frac{L \phi(x)}{\lambda M_{\phi}}$ as before.
This gives us
\begin{align*}
  \pi(x)T(x, y)
  &=
  \frac{\exp(U_{\neg i}(x))}{nZ} \; \mathbf{E}_r\Bigg[
    \frac{1}{\int \tilde{f}(u)du}\cdot
    \min\Bigg(\tilde{f}(y(i))\exp\left(
    \sum_{\phi\in A[i]}
    r_{\phi} \log\left( 1 + \bar \phi(x) \right)\right), 
    \\&\hspace{4em}\tilde{f}(x(i)) \exp\left(\sum_{\phi\in A[i]} r_{\phi}\log\left( 1+ \bar \phi(y)\right)\right)\Bigg) \Bigg]
  \\ &=
  \frac{\exp(U_{\neg i}(x))}{nZ} \; \mathbf{E}_r\left[
    \frac{1}{\int \tilde{f}(u)du}\cdot
    \min\left(\tilde{f}(y(i))\exp\left( U_{x(i)} \right), 
    \tilde{f}(x(i)) \exp\left(U_{y(i)}\right)\right) \right]
\end{align*}
where now the $\tilde f$ are considered to be a function of $r_\phi$ rather than $s_{\phi}$ as before.

To proceed further we will need to use the fact that $\tilde f$ is a Chebyshev interpolant to bound its error compared with $U$.
Recall that, here,
\[
    U_v = \sum_{\phi \in A[i]} r_{\phi} \log\left( 1 + \frac{L}{\lambda M_{\phi}} \phi(z_v) \right)
    = \sum_{\phi \in A[i]} r_{\phi} \log\left( 1 + \bar \phi(z_v) \right),
\]
and $\tilde f(v) \approx \exp(U_v)$ in the sense of being a degree-$m$ Chebyshev polynomial interpolant.
Recall that we assumed that the each function $\phi$, treated as a function in any single variable, must be analytic on a (shifted) Bernstein ellipse on the interval $[a,b]$ with parameter $\rho$ (i.e. a standard Bernstein ellipse on $[-1,1]$ with parameter $\rho$ shifted and scaled to have its foci at $a$ and $b$), and that its magnitude must be bounded by
\[
    \Abs{\phi(z)} \le M_{\phi}
\]
for any $z$ in this ellipse (keeping all the other parameters as usual within $[a,b]$.
It follows that the magnitude of the function $U_v$ is bounded by
\begin{align*}
    \Abs{ \exp(U_v) } 
    &= 
    \Abs{ \exp\left( \sum_{\phi \in A[i]} r_{\phi} \log\left( 1 + \frac{L}{\lambda M_{\phi}} \phi(z_v) \right) \right) } \\
    &= 
    \prod_{\phi \in A[i]} \Abs{ 1 + \frac{L}{\lambda M_{\phi}} \phi(z_v) }^{r_{\phi}} \\
    &\le
    \prod_{\phi \in A[i]} \left( 1 + \frac{L}{\lambda} \right)^{r_{\phi}}.
\end{align*}
Therefore, from Theorem \ref{thm:cheby}, we know that
\[
    \Abs{ \tilde f(v) - \exp(U_v) }
    \le
    \frac{4 \rho^{-m}}{\rho - 1} \cdot \prod_{\phi \in A[i]} \left( 1 + \frac{L}{\lambda} \right)^{r_{\phi}}
    =
    \frac{4 \rho^{-m}}{\rho - 1} \cdot \left( 1 + \frac{L}{\lambda} \right)^{\sum_{\phi \in A[i]} r_{\phi}}.
\]
Since we also assumed that $\phi(z)$ is always non-negative, $U_v$ must also be non-negative, and so in particular $\exp(-U_v) \le 1$, so
\[
    \Abs{ \frac{\tilde f(v)}{\exp(U_v)} - 1 } 
    \le
    \frac{4 \rho^{-m}}{\rho - 1} \cdot \left( 1 + \frac{L}{\lambda} \right)^{\sum_{\phi \in A[i]} r_{\phi}}
    \le
    \frac{4 \rho^{-m}}{\rho - 1} \cdot \exp\left( \frac{L}{\lambda} \sum_{\phi \in A[i]} r_{\phi} \right).
\]
If we now define
\[
    C = \frac{4 \rho^{-m}}{\rho - 1} \cdot \exp\left( \frac{L}{\lambda} \sum_{\phi \in A[i]} r_{\phi} \right),
\]
then
\[
    (1 - C) \cdot \exp(U_v) \le \tilde f(v) \le (1 + C) \cdot \exp(U_v).
\]
In particular, this means that
\[
    \min\left(\tilde{f}(y(i))\exp\left( U_{x(i)} \right), 
    \tilde{f}(x(i)) \exp\left(U_{y(i)}\right)\right)
    \ge
    (1 - C) \cdot \exp\left( U_{x(i)} + U_{y(i)} \right),
\]
and
\[
    \frac{1}{\int \tilde{f}(u) \; du}
    \ge
    \frac{1}{1 + C} \cdot \frac{1}{\int \exp(U_u) \; du}.
\]
Substituting this into our bound above gives
\begin{align*}
  \pi(x)T(x, y)
  &\ge
  \frac{\exp(U_{\neg i}(x))}{nZ} \; \mathbf{E}_r\left[ \frac{1 - C}{1 + C} \cdot
    \frac{\exp\left( U_{x(i)} + U_{y(i)} \right)}{\int \exp(U_u) \; du} \right].
\end{align*}
Now, recall that we set this up by sampling $r_{\phi}$ independently from a Poisson random variable $r_{\phi} \sim \text{Poisson}\left( \frac{\lambda M_{\phi}}{L} \right)$.
This distribution is equivalent to assigning
\[
    \Lambda =  \sum_{\phi \in A[i]} \frac{\lambda M_{\phi}}{L},
\]
sampling the random variable $B \sim \text{Poisson}\left( \Lambda \right)$, and then sampling $r_{\phi} \sim \operatorname{Multinomial}\left(B, \frac{\lambda M_{\phi}}{\Lambda L} \right)$.
If we re-think our distribution as coming from this process, then by the Law of Total Expectation,
\begin{align*}
  \pi(x)T(x, y)
  &\ge
  \frac{\exp(U_{\neg i}(x))}{nZ} \; \mathbf{E}_B\left[ \frac{1 - C}{1 + C}  \cdot \mathbf{E}_r\left[
    \frac{\exp\left( U_{x(i)} + U_{y(i)} \right)}{\int \exp(U_u) \; du} \middle| B \right] \right],
\end{align*}
where we can pull out the terms in $C$ because we can write $C$ to depend only on $B$ as
\[
    C 
    =
    \frac{4 \rho^{-m}}{\rho - 1} \cdot \exp\left( \frac{L}{\lambda} \sum_{\phi \in A[i]} r_{\phi} \right)
    =
    \frac{4 \rho^{-m}}{\rho - 1} \cdot \exp\left( \frac{L B}{\lambda} \right).
\]
Next, we can bound this inner expectation with
\begin{align*}
    &\mathbf{E}_r\left[
    \frac{\exp\left( U_{x(i)} + U_{y(i)} \right)}{\int \exp(U_u) \; du} \middle| B \right] \\
    &=
    \mathbf{E}_r\left[
    \frac{1}{\int \exp(U_u - U_{x(i)} - U_{y(i)}) \; du} \middle| B \right] \\
    &\ge
    \mathbf{E}_r\left[
    \int \exp(U_u - U_{x(i)} - U_{y(i)}) \; du \middle| B \right]^{-1} \\
    &=
    \mathbf{E}_r\left[
    \int 
        \exp\left(\sum_{\phi \in A[i]} r_{\phi} \left(
            \log\left( 1 + \bar \phi(z_u) \right)
            -
            \log\left( 1 + \bar \phi(x) \right)
            -
            \log\left( 1 + \bar \phi(y) \right)
        \right) \right)
    \; du \middle| B \right]^{-1} \\
    &=
    \mathbf{E}_r\left[
    \int 
        \exp\left(\sum_{\phi \in A[i]} r_{\phi} t_{\phi} \right)
    \; du \middle| B \right]^{-1} \\
    &=
    \left( \int \mathbf{E}_r\left[
        \exp\left(\sum_{\phi \in A[i]} r_{\phi} t_{\phi} \right)
    \middle| B \right]  \; du \right)^{-1},
\end{align*}
where we define
\[
    t_{\phi} 
    = 
    \log\left( 1 + \bar \phi(z_u) \right)
    -
    \log\left( 1 + \bar \phi(x) \right)
    -
    \log\left( 1 + \bar \phi(y) \right).
\]
This inner expectation is now just the moment-generating function of the multinomial distribution.
Applying the standard formula for that MGF gives us
\[
    \mathbf{E}_r\left[
        \exp\left(\sum_{\phi \in A[i]} r_{\phi} t_{\phi} \right)
    \middle| B \right]
    =
    \left(
        \sum_{\phi \in A[i]} \frac{\lambda M_{\phi}}{\Lambda L} \cdot \exp(t_{\phi})
    \right)^B.
\]
Substituting this back into our original expression gives
\begin{align*}
  \pi(x)T(x, y)
  &\ge
  \frac{\exp(U_{\neg i}(x))}{nZ} \; \mathbf{E}_B\left[ \frac{1 - C}{1 + C}  \cdot \left( \int \left(
        \sum_{\phi \in A[i]} \frac{\lambda M_{\phi}}{\Lambda L} \cdot \exp(t_{\phi})
    \right)^B  \; du \right)^{-1} \right].
\end{align*}
Next, let $\delta > 0$ be a small constant, to be assigned later.
Recall that for any non-negative random variable $X$ and any event $A$, by the Law of Total Probability,
\[
    \Exv{X}
    =
    \Exv{X | A} \cdot \mathbf{P}(A) + \Exv{X | \neg A} \cdot \mathbf{P}(\neg A)
    \ge
    \Exv{X | A} \cdot \mathbf{P}(A).
\]
So, since the interior of this expectation is a non-negative number, it follows that
\begin{align*}
  \pi(x)T(x, y)
  &\ge
  \frac{\exp(U_{\neg i}(x))}{nZ} \; \mathbf{E}_B\left[ \frac{1 - C}{1 + C}  \cdot \left( \int \left(
        \sum_{\phi \in A[i]} \frac{\lambda M_{\phi}}{\Lambda L} \cdot \exp(t_{\phi})
    \right)^B  \; du \right)^{-1} \middle| C \le \delta \right]
    \\&\hspace{4em}\cdot \mathbf{P}_B(C \le \delta) \\
  &\ge
  \frac{\exp(U_{\neg i}(x))}{nZ} \cdot \frac{1 - \delta}{1 + \delta} \cdot \mathbf{E}_B\left[ \left( \int \left(
        \sum_{\phi \in A[i]} \frac{\lambda M_{\phi}}{\Lambda L} \cdot \exp(t_{\phi})
    \right)^B  \; du \right)^{-1} \middle| C \le \delta \right]
     \\&\hspace{4em}\cdot \mathbf{P}_B(C \le \delta).
\end{align*}
By Jensen's inequality again, we get
\begin{align*}
  \pi(x)T(x, y)
  &\ge
  \frac{\exp(U_{\neg i}(x))}{nZ} \; \mathbf{E}_B\left[ \frac{1 - C}{1 + C}  \cdot \left( \int \left(
        \sum_{\phi \in A[i]} \frac{\lambda M_{\phi}}{\Lambda L} \cdot \exp(t_{\phi})
    \right)^B  \; du \right)^{-1} \middle| C \le \delta \right]
    \\&\hspace{4em}\cdot \mathbf{P}_B(C \le \delta) \\
  &\ge
  \frac{\exp(U_{\neg i}(x))}{nZ} \cdot \frac{1 - \delta}{1 + \delta} \cdot  \left( \int \mathbf{E}_B\left[\left(
        \sum_{\phi \in A[i]} \frac{\lambda M_{\phi}}{\Lambda L} \cdot \exp(t_{\phi})
    \right)^B \middle| C \le \delta \right] \; du \right)^{-1}
     \\&\hspace{4em}\cdot \mathbf{P}_B(C \le \delta).
\end{align*}
Since this inner expectation is again non-negative, we can again apply our above inequality, but in the opposite direction, giving
\[
    \Exv{X | A} \le \frac{\Exv{X}}{\mathbf{P}(A)}.
\]
This produces
\begin{align*}
  \pi(x)T(x, y)
  &\ge
  \frac{\exp(U_{\neg i}(x))}{nZ} \cdot \frac{1 - \delta}{1 + \delta} \cdot  \left( \int \mathbf{E}_B\left[\left(
        \sum_{\phi \in A[i]} \frac{\lambda M_{\phi}}{\Lambda L} \cdot \exp(t_{\phi})
    \right)^B \right] \; du \right)^{-1}
     \\&\hspace{4em}\cdot \mathbf{P}_B(C \le \delta)^2.
\end{align*}
Now, we are just left with the MGF of a Poisson-distributed random variable.
This we already know to be
\begin{align*}
    \mathbf{E}_B\left[\left(
        \sum_{\phi \in A[i]} \frac{\lambda M_{\phi}}{\Lambda L} \cdot \exp(t_{\phi})
    \right)^B \right]
    &=
    \mathbf{E}_B\left[ \exp\left( B \log\left(
        \sum_{\phi \in A[i]} \frac{\lambda M_{\phi}}{\Lambda L} \cdot \exp(t_{\phi})
    \right) \right) \right] \\
    &=
    \exp\left( \Lambda 
    \left( \left( \sum_{\phi \in A[i]} \frac{\lambda M_{\phi}}{\Lambda L} \cdot \exp(t_{\phi}) \right) - 1 \right) \right) \\
    &=
    \exp\left( \sum_{\phi \in A[i]} \frac{\lambda M_{\phi}}{L} \cdot \left( \exp(t_{\phi}) - 1 \right) \right),
\end{align*}
where in the last line we can leverage the fact that
\[
    \sum_{\phi \in A[i]} \frac{\lambda M_{\phi}}{\Lambda L} = 1
\]
to justify pulling the $-1$ inside the sum.
From the analysis of Poisson-Gibbs, we had that
\[
    \exp(t_{\phi}) - 1
    \le
    \bar \phi(z_u) - \bar \phi(x) - \bar \phi(y) + \frac{4L^2}{\lambda^2}.
\]
So,
\begin{align*}
    \mathbf{E}_B\left[\left(
        \sum_{\phi \in A[i]} \frac{\lambda M_{\phi}}{\Lambda L} \cdot \exp(t_{\phi})
    \right)^B \right]
    &\le
    \exp\left( \sum_{\phi \in A[i]} \frac{\lambda M_{\phi}}{L} \cdot \left( 
        \bar \phi(z_u) - \bar \phi(x) - \bar \phi(y) + \frac{4L^2}{\lambda^2}
    \right) \right) \\
    &=
    \exp\left( \sum_{\phi \in A[i]} \left( 
        \phi(z_u) - \phi(x) - \phi(y) + \frac{4L M_{\phi}}{\lambda}
    \right) \right) \\
    &\le
    \exp\left( \bar U_u - \bar U_{x(i)} - \bar U_{y(i)} + \frac{4L^2}{\lambda} \right),
\end{align*}
where as in the analysis of Poisson-Gibbs, $\bar U_v$ denotes the assignment of $U_v$ in the plain Gibbs sampling algorithm (Algorithm~\ref{alg:gibbs}),
\[
    \bar U_v = \sum_{\phi \in A[i]} \phi(z_v).
\]
Substituting this expression in to our overall bound, we get
\begin{align*}
  \pi(x)T(x, y)
  &\ge
  \frac{\exp(U_{\neg i}(x))}{nZ} \cdot \frac{1 - \delta}{1 + \delta} \cdot  \left( \int  \exp\left( \bar U_u - \bar U_{x(i)} - \bar U_{y(i)} + \frac{4L^2}{\lambda} \right) \; du \right)^{-1}
     \\&\hspace{4em}\cdot \mathbf{P}_B(C \le \delta)^2 \\
  &=
  \frac{\exp(U(x))}{nZ} \cdot \frac{1 - \delta}{1 + \delta} \cdot \frac{\exp(\bar U_{y(i)})}{\int \exp\left( \bar U_u \right) \; du}
     \\&\hspace{4em}\cdot \exp\left(-\frac{4L^2}{\lambda}\right) \cdot \mathbf{P}_B(C \le \delta)^2.
\end{align*}
Finally, if we let $G$ denote the transition probability operator of plain Gibbs sampling, we notice right away that
\begin{align*}
  \pi(x)T(x, y)
  &\ge
  \frac{1 - \delta}{1 + \delta}
  \cdot \exp\left(-\frac{4L^2}{\lambda}\right) \cdot \mathbf{P}_B(C \le \delta)^2 \cdot \pi(x) G(x,y) \\
  &\ge
  (1 - 2\delta) \cdot \exp\left(-\frac{4L^2}{\lambda}\right) \cdot \mathbf{P}_B(C \le \delta)^2 \cdot \pi(x) G(x,y).
\end{align*}
To get a final bound, all we need to do is bound $\mathbf{P}_B(C \le \delta)$.
This is straightforward, since
\begin{align*}
    \mathbf{P}_B(C \le \delta)
    &=
    \mathbf{P}_B\left(\frac{4 \rho^{-m}}{\rho - 1} \cdot \exp\left( \frac{L B}{\lambda} \right) \le \delta \right) \\
    &=
    \mathbf{P}_B\left( \exp\left( \frac{L B}{\lambda} \right) \le \frac{\rho - 1}{4 \rho^{-m}} \cdot \delta \right).
\end{align*}
Notice that by the MGF formula for $B$,
\[
    \mathbf{E}_B\left[ \exp\left( \frac{L B}{\lambda} \right) \right] 
    \le
    \exp\left(\Lambda \left(\exp\left( \frac{L}{\lambda} \right) - 1 \right) \right).
\]
Since we chose a minibatch size parameter $\lambda \ge 2 L$, it follows that $L / \lambda \le 1/2$, and so
\[
    \exp\left( \frac{L}{\lambda} \right) - 1 \le \frac{2L}{\lambda},
\]
and so since also
\[
    \Lambda = \sum_{\phi \in A[i]} \frac{\lambda M_{\phi}}{L} \le \lambda.
\]
it follows that
\[
    \mathbf{E}_B\left[ \exp\left( \frac{L B}{\lambda} \right) \right] 
    \le
    \exp\left(\lambda \cdot \frac{2L}{\lambda} \right)
    =
    \exp(2 L).
\]
Therefore, by Markov's inequality,
\begin{align*}
    \mathbf{P}_B(C \ge \delta)
    &=
    \mathbf{P}_B\left( \exp\left( \frac{L B}{\lambda} \right) \ge \frac{\rho - 1}{4 \rho^{-m}} \cdot \delta \right) \\
    &\le
    \frac{\exp(2L)}{\frac{\rho - 1}{4 \rho^{-m}} \cdot \delta} \\
    &\le
    \frac{4 \rho^{-m}}{\rho - 1} \cdot \frac{\exp(2L)}{\delta}.
\end{align*}
Thus,
\begin{align*}
    \mathbf{P}_B(C \le \delta)
    &=
    1 - \mathbf{P}_B(C \ge \delta) \\
    &\ge
    1 - \frac{4 \rho^{-m}}{\rho - 1} \cdot \frac{\exp(2L)}{\delta},
\end{align*}
and in particular
\begin{align*}
    \mathbf{P}_B(C \le \delta)^2
    &=
    \left( 1 - \mathbf{P}_B(C \ge \delta) \right)^2 \\
    &\ge
    1 - 2 \mathbf{P}_B(C \ge \delta) \\
    &\ge
    1 - \frac{8 \rho^{-m}}{\rho - 1} \cdot \frac{\exp(2L)}{\delta}.
\end{align*}
Substituting this back into our overall bound gives us
\begin{align*}
  \pi(x)T(x, y)
  &\ge
  \frac{1 - \delta}{1 + \delta}
  \cdot \exp\left(-\frac{4L^2}{\lambda}\right) \cdot \mathbf{P}_B(C \le \delta)^2 \cdot \pi(x) G(x,y) \\
  &\ge
  (1 - 2 \delta)
  \cdot 
  \left(1 - \frac{8 \rho^{-m}}{\rho - 1} \cdot \frac{\exp(2L)}{\delta} \right)
  \cdot
  \exp\left(-\frac{4L^2}{\lambda}\right) \cdot \pi(x) G(x,y) \\
  &\ge
  \left(1 - 2 \delta - \frac{8 \rho^{-m}}{\rho - 1} \cdot \frac{\exp(2L)}{\delta} \right)
  \cdot
  \exp\left(-\frac{4L^2}{\lambda}\right) \cdot \pi(x) G(x,y).
\end{align*}
Finally, choosing the value of $\delta$ as
\[
    \delta = \frac{2 \exp(L)}{\rho^{m/2} \cdot \sqrt{\rho - 1}},
\]
we get
\begin{align*}
  \pi(x)T(x, y)
  &\ge
  \left(1 - \frac{8 \exp(L) \rho^{-m/2}}{\sqrt{\rho - 1}} \right)
  \cdot
  \exp\left(-\frac{4L^2}{\lambda}\right) \cdot \pi(x) G(x,y).
\end{align*}
Now applying the standard Dirichlet form argument, we get
\[
    \bar \gamma
    \ge
    \left(1 - \frac{8 \exp(L) \rho^{-m/2}}{\sqrt{\rho - 1}} \right)
    \cdot
    \exp\left(-\frac{4L^2}{\lambda}\right) \cdot \gamma,
\]
which was the desired expression.
\end{proof}

\subsection{Proof of Theorem \ref{thm:doupoly}}
\begin{proof}
The reversibility can be proved by the same procedure as in Section~\ref{sec:proof-thm-poly}.
By applying that same analysis, which did not depend on the manner in which the approximation $\tilde f$ was constructed, we can arrive at the expression
\begin{align*}
  \pi(x)T(x, y)
  &=
  \frac{\exp(U_{\neg i}(x))}{nZ} \; \mathbf{E}_r\left[
    \frac{1}{\int \tilde{f}(u)du}\cdot
    \min\left(\tilde{f}(y(i))\exp\left( U_{x(i)} \right), 
    \tilde{f}(x(i)) \exp\left(U_{y(i)}\right)\right) \right].
\end{align*}
By the assumption of $\phi(z)$, we have
\begin{align*}
\Abs{U_v} &=\Abs{ \sum_{\phi \in A[i]} r_{\phi} \log\left( 1 + \bar \phi(z_v) \right)}\\
&\le \sum_{\phi \in A[i]} r_{\phi} \Abs{\log\left( 1 + \frac{L}{\lambda M_{\phi}}\phi(x) \right)}\\
&\le \sum_{\phi \in A[i]} r_{\phi} \Abs{\frac{2L}{\lambda M_{\phi}}\phi(x)}\\
&\le \frac{2L}{\lambda} \sum_{\phi \in A[i]} r_{\phi}.
\end{align*}
where the second inequality holds because
\[
    \Abs{z}\le \frac{1}{2}
    \hspace{1em}\Rightarrow\hspace{1em}
    \Abs{\log(1 + z)}\le 2\Abs{z},
\]
using the assumptions $\lambda\ge 2L$ and $\Abs{\phi(x)}\le M_{\phi}$.
Now applying Lemma \ref{thm:chebyext} in Section~\ref{sec:chebyext}, assigning $\sigma=\sqrt{\rho}$ gives us,
\begin{align*}
\Abs{\tilde{U}_v - U_v} 
&\le \frac{8\rho^{-\frac{m}{2}}}{\sqrt{\rho} - 1}\cdot \frac{L}{\lambda} \sum_{\phi \in A[i]} r_{\phi},
\end{align*}
for any $v$ in the shifted-and-scaled Bernstein ellipse with parameter $\sqrt{\rho}$.

Next, since $\tilde{U}_v$ is a polynomial in $v$, $\exp(\tilde{U}_v)$ must be analytic everywhere in $\mathbb{C}$. In particular it must be analytic on the Bernstein ellipse on the interval $[a, b]$ with parameter $\sqrt{\rho}$.
On that interval, it is bounded by
\begin{align*}
\Abs{\exp(\tilde{U}_v)} 
&\le \exp\left(\Abs{\tilde{U}_v}\right)\\
&\le \exp\left(\Abs{U_v} + \Abs{\tilde U_v - U_v}\right)\\
&\le \exp\left(\frac{2L}{\lambda}\sum_{\phi \in A[i]} r_{\phi}\right) \cdot \exp\left(\frac{8\rho^{-\frac{m}{2}}}{\sqrt{\rho} - 1}\cdot \frac{L}{\lambda} \sum_{\phi \in A[i]} r_{\phi}\right)\\
&\le \exp\left(\frac{4\rho^{-\frac{m}{2}}+ \sqrt{\rho} - 1}{\sqrt{\rho} - 1}\cdot \frac{2L}{\lambda} \sum_{\phi \in A[i]} r_{\phi}\right).
\end{align*}
Now applying Theorem \ref{thm:cheby} using the Bernstein ellipse with parameter $\sqrt{\rho}$, we have, for any $v$ on the interval $[a,b]$,
\[
\Abs{\tilde{f}(v) - \exp(\tilde{U}_v)}\le \frac{4\rho^{-\frac{k}{2}}}{\sqrt{\rho} - 1}\cdot \exp\left(\frac{4\rho^{-\frac{m}{2}}+ \sqrt{\rho} - 1}{\sqrt{\rho} - 1}\cdot \frac{2L}{\lambda} \sum_{\phi \in A[i]} r_{\phi}\right)
\]
Therefore, it follows that
\begin{align*}
\Abs{\frac{\tilde{f}(v)}{\exp(U_v)} - 1}
&\le
\Abs{\frac{\tilde{f}(v) - \exp(\tilde U_v) + \exp(\tilde U_v)}{\exp(U_v)} - 1}  \\
&\le
\frac{\Abs{\tilde{f}(v) - \exp(\tilde U_v)}}{\exp(U_v)}
+
\Abs{\exp(\tilde U_v - U_v) - 1}   \\
&\le
\Abs{\tilde{f}(v) - \exp(\tilde U_v)}
+
\exp\left(\Abs{\tilde U_v - U_v}\right) - 1,
\end{align*}
where the last inequality is justified by the fact that $U_v$ is non-negative and for any $x$, $\Abs{\exp(x) - 1} \le \exp(\Abs{x}) - 1$.
Now substituting in our bounds from above gives us
\begin{align*}
&\hspace{-1em}\Abs{\frac{\tilde{f}(v)}{\exp(U_v)} - 1} \\
&\le \exp\left(\frac{8\rho^{-\frac{m}{2}}}{\sqrt{\rho} - 1}\cdot \frac{L}{\lambda} \sum_{\phi \in A[i]} r_{\phi}\right) + \frac{4\rho^{-\frac{k}{2}}}{\sqrt{\rho} - 1}\cdot \exp\left(\frac{4\rho^{-\frac{m}{2}}+ \sqrt{\rho} - 1}{\sqrt{\rho} - 1}\cdot \frac{2L}{\lambda} \sum_{\phi \in A[i]} r_{\phi}\right) - 1
\end{align*}
As before, we let $B = \sum_{\phi\in A[i]}r_{\phi}$ where $B\sim \text{Poisson}(\Lambda)$. Then 
\begin{align*}
\Abs{\frac{\tilde{f}(v)}{\exp(U_v)} - 1} 
&\le \exp\left(\frac{8\rho^{-\frac{m}{2}}}{\sqrt{\rho} - 1}\cdot \frac{LB}{\lambda} \right) + \frac{4\rho^{-\frac{k}{2}}}{\sqrt{\rho} - 1}\cdot \exp\left(\frac{4\rho^{-\frac{m}{2}}+ \sqrt{\rho} - 1}{\sqrt{\rho} - 1}\cdot \frac{2LB}{\lambda}\right) - 1
\end{align*}
We define
\[E = \exp\left(\frac{8\rho^{-\frac{m}{2}}}{\sqrt{\rho} - 1}\cdot \frac{LB}{\lambda} \right) + \frac{4\rho^{-\frac{k}{2}}}{\sqrt{\rho} - 1}\cdot \exp\left(\frac{4\rho^{-\frac{m}{2}}+ \sqrt{\rho} - 1}{\sqrt{\rho} - 1}\cdot \frac{2LB}{\lambda}\right) - 1,\]
and by following the same steps as used in Section~\ref{sec:proof-thm-poly}, with $E$ in place of the $C$ of that proof, we can get, for any constant $\delta > 0$,
\begin{align*}
  \pi(x)T(x, y)
  &\ge
  (1 - 2\delta) \cdot \exp\left(-\frac{4L^2}{\lambda}\right) \cdot \mathbf{P}_B(E \le \delta)^2 \cdot \pi(x) G(x,y).
\end{align*}
All that remains is to bound $\mathbf{P}_B(E \le \delta)$.
Using the MGF formula for $B$ twice, we get that 
\begin{align*}
  \mathbf{E}_B(E) &=  \frac{4\rho^{-\frac{k}{2}}}{\sqrt{\rho} - 1}\cdot \exp\left(\Lambda\left(\exp\left(\frac{4\rho^{-\frac{m}{2}}+ \sqrt{\rho} - 1}{\sqrt{\rho} - 1}\cdot \frac{2L}{\lambda}\right) - 1\right) \right)
  \\&\hspace{2em} + \exp\left(\Lambda\left(\exp\left(\frac{8\rho^{-\frac{m}{2}}}{\sqrt{\rho} - 1}\cdot\frac{L}{\lambda}\right) - 1\right) \right) - 1.
\end{align*}
If we require that $m$ is large enough that
\[
    4\rho^{-\frac{m}{2}} \le \sqrt{\rho} - 1,
\]
then
\begin{align*}
  \mathbf{E}_B(E) &\le  \frac{4\rho^{-\frac{k}{2}}}{\sqrt{\rho} - 1}\cdot \exp\left(\Lambda\left(\exp\left( \frac{4L}{\lambda}\right) - 1\right) \right)
  \\&\hspace{2em} + \exp\left(\Lambda\left(\exp\left(\frac{8\rho^{-\frac{m}{2}}}{\sqrt{\rho} - 1}\cdot\frac{L}{\lambda}\right) - 1\right) \right) - 1.
\end{align*}
By Taylor's theorem, for $x > 0$,
\[
    \exp(x) - 1 = \exp(x) - \exp(0) \le x \cdot \exp(x).
\]
So, since $\Lambda \le \lambda$, we can bound our expectation with
\begin{align*}
  \mathbf{E}_B(E) &\le \frac{4\rho^{-\frac{k}{2}}}{\sqrt{\rho} - 1}\cdot \exp\left(\Lambda \cdot \frac{4L}{\lambda} \cdot \exp\left( \frac{4L}{\lambda}\right) \right)
  \\&\hspace{2em} + \exp\left(\Lambda \cdot \frac{8\rho^{-\frac{m}{2}}}{\sqrt{\rho} - 1}\cdot\frac{L}{\lambda} \cdot \exp\left(\frac{8\rho^{-\frac{m}{2}}}{\sqrt{\rho} - 1}\cdot\frac{L}{\lambda}\right) \right) - 1
  \\&\le
  \frac{4\rho^{-\frac{k}{2}}}{\sqrt{\rho} - 1}\cdot \exp\left(4L \cdot \exp\left( \frac{4L}{\lambda}\right) \right)
  \\&\hspace{2em} + \exp\left( \frac{8\rho^{-\frac{m}{2}}}{\sqrt{\rho} - 1}\cdot L \cdot \exp\left(\frac{8\rho^{-\frac{m}{2}}}{\sqrt{\rho} - 1}\cdot\frac{L}{\lambda}\right) \right) - 1
  \\&\le
  \frac{4\rho^{-\frac{k}{2}}}{\sqrt{\rho} - 1}\cdot \exp\left(4L \cdot \exp\left( \frac{4L}{\lambda}\right) \right)
  \\&\hspace{2em} + \exp\left( \frac{8\rho^{-\frac{m}{2}}}{\sqrt{\rho} - 1}\cdot L \cdot \exp\left(\frac{4L}{\lambda}\right) \right) - 1.
\end{align*}
Since $\lambda \log(2) \ge 4 L$, we can bound $\exp(4L/\lambda) \le 2$, and so
\begin{align*}
  \mathbf{E}_B(E) &\le
  \frac{4\rho^{-\frac{k}{2}}}{\sqrt{\rho} - 1}\cdot \exp\left(8 L \right)
  + \exp\left( \frac{16 L \rho^{-\frac{m}{2}}}{\sqrt{\rho} - 1} \right) - 1.
\end{align*}
We now define
\[
    F
    =
    \frac{4\cdot \exp\left(8 L \right)\cdot\rho^{-\frac{k}{2}}}{\sqrt{\rho} - 1}
  + \exp\left( \frac{16 L \rho^{-\frac{m}{2}}}{\sqrt{\rho} - 1} \right) - 1.
\]
By Markov’s inequality,
\begin{align*}
  &\hspace{-0em}\mathbf{P}_B(E \ge \delta)\ge \frac{\mathbf{E}_B(E)}{\delta}\ge  F/\delta.
\end{align*}
It follows
\begin{align*}
  \mathbf{P}_B(E \le \delta)^2 = \left(1 - \mathbf{P}_B(E \ge \delta)\right)^2\ge 1 - 2\mathbf{P}_B(E \ge \delta)\ge 1 -  2F/\delta.
\end{align*}
Substituting it back into the overall bound,
\begin{align*}
  \pi(x)T(x, y)
  &\ge
  (1 - 2\delta) \cdot \exp\left(-\frac{4L^2}{\lambda}\right) \cdot \mathbf{P}_B(E \le \delta)^2 \cdot \pi(x) G(x,y)\\
  &\ge 
  \left(1 - 2\delta -  \frac{2F}{\delta}\right)\cdot\exp\left(-\frac{4L^2}{\lambda}\right)\cdot \pi(x) G(x,y)
\end{align*}
Let 
\[
\delta = \sqrt{F},
\]
it becomes
\begin{align*}
  \pi(x)T(x, y)
  &\ge
  \left(1 - 4\sqrt{F}\right)\cdot\exp\left(-\frac{4L^2}{\lambda}\right)\cdot \pi(x) G(x,y)
\end{align*}
Again, using the Dirichlet form we bound the spectral gap,
\begin{align*}
 \bar{\gamma}
&\geq
\left(1 - 4\sqrt{F}\right)\exp\left(\frac{-4L^2}{\lambda}\right)\cdot \gamma
\end{align*}
\end{proof}

\section{Poisson-MH}\label{sec:MH}
We apply our Poisson-minibatching method to Metropolis-Hasting sampling. In Poisson-minibatching M-H (Poisson-MH), we first generate a candidate $x^*$ from the proposal distribution $q(x^*|x)$. Then the M-H ratio will be calculated as following
\begin{align*}
p = \frac{\exp\left(\sum_{\phi \in S} s_{\phi} \log\left( 1 + \frac{L}{\lambda M_{\phi}}\phi(x^*) \right)\right)q(x^*|x)}{\exp\left(\sum_{\phi \in S} s_{\phi} \log\left( 1 + \frac{L}{\lambda M_{\phi}}\phi(x) \right)\right)q(x|x^*)}
\end{align*}

We accept $x^*$ with the probability $\min(1, p)$. After applying Poisson-minibatching, the M-H ratio no longer needs to use the whole dataset which will reduce the computational cost significantly. 

Theorem~\ref{thm:mh} is similar to the bounds of Poisson-Gibbs. As long as we set $\lambda = \Theta(L^2)$, the convergence is slowed down by at most a constant factor which is unrelated to the size of the  problem.
\subsection{Proof of Theorem~\ref{thm:mh}}
\begin{proof}
We begin with the transition probability from $x$ to $x^*$
\begin{align*}
  &T( x^*,  x)\\
  &=
  \mathbf{E}\left\{q( x^*| x)\min\left(1, \frac{q( x| x^*)\pi( x^*, s)}{q( x^*| x)\pi( x, s)}\right)\right\}\\
  &=
\mathbf{E}\left\{q( x^*| x)\min\left(1, \frac{q( x| x^*)\exp\left(  \sum_{\phi\in \Phi} \left[s_{\phi}\log\left( \frac{\lambda M_{\phi}}{L}+ \phi( x^*)\right) - \log s_{\phi}!\right]\right)}
{q( x^*| x)\exp\left(  \sum_{\phi\in \Phi} \left[s_{\phi}\log\left( \frac{\lambda M_{\phi}}{L}+ \phi( x)\right) - \log s_{\phi}!\right]\right)}\right)\right\}\\
&=
\mathbf{E}\left\{q( x^*| x)\min\left(1, \frac{q( x| x^*)\exp\left(  \sum_{\phi\in \Phi} \left[s_{\phi}\log\left( \frac{\lambda M_{\phi}}{L}+ \phi( x^*)\right)\right]\right)}
{q( x^*| x)\exp\left(  \sum_{\phi\in \Phi} \left[s_{\phi}\log\left( \frac{\lambda M_{\phi}}{L}+ \phi( x)\right) \right]\right)}\right)\right\}\\
&=
\sum_s\left\{q( x^*| x)\min\left(1, \frac{q( x| x^*)\exp\left(  \sum_{\phi\in \Phi} \left[s_{\phi}\log\left( \frac{\lambda M_{\phi}}{L}+ \phi( x^*)\right)\right]\right)}
{q( x^*| x)\exp\left(  \sum_{\phi\in \Phi} \left[s_{\phi}\log\left( \frac{\lambda M_{\phi}}{L}+ \phi( x)\right) \right]\right)}\right)\right\}\prod_{\phi \in \Phi} p(s_{\phi}| x)\\
&=
\sum_s\left\{q( x^*| x)\min\left(\exp\left(\sum_{\phi\in \Phi} \left[s_{\phi}\log\left( \frac{\lambda M_{\phi}}{L}+ \phi( x)\right) 
- \phi( x) -  \frac{\lambda M_{\phi}}{L}- \log s_{\phi}!\right] \right),\right.\right.\\
&\hspace{2em}\left.\left.\frac{q( x| x^*)\exp\left(  \sum_{\phi\in \Phi} \left[s_{\phi}\log\left( \frac{\lambda M_{\phi}}{L}+ \phi( x^*)\right)\right]\right)}
{q( x^*| x)\exp\left(  \sum_{\phi\in \Phi}\phi( x) +  \frac{\lambda M_{\phi}}{L}+ \log s_{\phi}! \right)}\right)\right\}\\
&=
\sum_s\left\{q( x^*| x)\min\left(\exp\left(\sum_{\phi\in \Phi} \left[s_{\phi}\log\left( \frac{\lambda M_{\phi}}{L}+ \phi( x)\right) 
- \phi( x) -  \frac{\lambda M_{\phi}}{L}- \log s_{\phi}!\right] \right),\right.\right.\\
&\hspace{2em}\left.\left.\frac{q( x| x^*) }
{q( x^*| x) }\exp\left(\sum_{\phi\in \Phi} \left[s_{\phi}\log\left( \frac{\lambda M_{\phi}}{L}+ \phi( x^*)\right) 
- \phi( x) -  \frac{\lambda M_{\phi}}{L}- \log s_{\phi}!\right] \right)\right)\right\}
\end{align*}

Multiplying $\pi(x)$ to both sides,
\begin{align*}
&\pi( x)T( x^*,  x)\\
  &=
\frac{1}{Z}\exp\left(\sum_{\phi\in \Phi}\phi( x) \right)T( x^*,  x)\\
  &=
\frac{1}{Z}\sum_s
\min\Bigg(q( x^*| x) \left(\exp\left(\sum_{\phi\in \Phi} \left[s_{\phi}\log\left( \frac{\lambda M_{\phi}}{L}+ \phi( x)\right)  -  \frac{\lambda M_{\phi}}{L}- \log s_{\phi}!\right] \right),\right.\\
&\hspace{2em}\left.q( x| x^*) 
\exp\left(\sum_{\phi\in \Phi} \left[s_{\phi}\log\left( \frac{\lambda M_{\phi}}{L}+ \phi( x^*)\right) -  \frac{\lambda M_{\phi}}{L}- \log s_{\phi}!\right] \right)\right)\Bigg)  
\end{align*}

This implies the Markov chain is reversible.

We can continue to reduce this to
\begin{align*}
  &\pi( x) T( x^*,  x)\\
  &=
\frac{1}{Z}\sum_s
\min\left(q( x^*| x) \exp\left(\sum_{\phi\in \Phi} s_{\phi}\left[\log\left( \frac{\lambda M_{\phi}}{L}+ \phi( x)  \right)- \log \frac{\lambda M_{\phi}}{L}\right] \right),\right.\\
&\hspace{2em}\left.q( x| x^*) 
\exp\left(\sum_{\phi\in \Phi} s_{\phi}\left[\log\left( \frac{\lambda M_{\phi}}{L}+ \phi( x^*) \right)- \log \frac{\lambda M_{\phi}}{L}\right] \right)\right)
\\&\hspace{2em}\cdot \prod_{\phi \in \Phi} \frac{1}{s_{\phi}!}\exp\left(- \frac{\lambda M_{\phi}}{L}\right)\left( \frac{\lambda M_{\phi}}{L}\right)^{s_{\phi}}\\
  &=
\frac{1}{Z}\sum_s
\min\left(q( x^*| x) \exp\left(\sum_{\phi\in \Phi} s_{\phi}\log\left( 1+ \frac{L}{\lambda M_{\phi}}\phi( x)  \right) \right),\right.\\
&\hspace{2em}\left.q( x| x^*) 
\exp\left(\sum_{\phi\in \Phi} s_{\phi}\log\left( 1+ \frac{L}{\lambda M_{\phi}}\phi( x^*) \right) \right)\right)
\cdot \prod_{\phi \in \Phi} \frac{1}{s_{\phi}!}\exp\left(- \frac{\lambda M_{\phi}}{L}\right)\left( \frac{\lambda M_{\phi}}{L}\right)^{s_{\phi}}  
\end{align*}

Similar to the previous proof, $s_{\phi}$ here are non-negative integers that a Poisson variable can take, not variables. So if we let $r_{\phi} \sim \text{Poisson}\left( \frac{\lambda M_{\phi}}{L} \right)$ and $r_{\phi}$ to be all independent, we can write this as
\begin{align*}
  \pi( x) T( x^*,  x)
  &=
\frac{1}{Z}\mathbf{E}
\min\left(q( x^*| x) \exp\left(\sum_{\phi\in \Phi} r_{\phi}\log\left( 1+ \frac{L}{\lambda M_{\phi}}\phi( x) \right) \right),\right.\\
&\left.q( x| x^*) 
\exp\left(\sum_{\phi\in \Phi} r_{\phi}\log\left( 1+ \frac{L}{\lambda M_{\phi}}\phi( x^*)\right)\right)\right)
\end{align*}

Assume $G( x^*,  x)$ is the transition operator of a plain MCMC. Consider the ratio,
\begin{align*}
\frac{\pi( x) T( x^*,  x)}
  {\pi( x) G( x^*,  x)}
&=
\frac{1}{Z}\mathbf{E}
\min\Bigg(q( x^*| x)\exp\left(\sum_{\phi\in \Phi} r_{\phi}\log\left( 1+ \frac{L}{\lambda M_{\phi}}\phi( x) \right)  \right),
\\&\hspace{2em}q( x| x^*)\exp\left(\sum_{\phi\in \Phi} r_{\phi}\log\left( 1+ \frac{L}{\lambda M_{\phi}}\phi( x^*)\right) \right)\Bigg)
\\&\hspace{2em}\cdot\Bigg[1\bigg/\Bigg(\frac{1}{Z}
\min\left(q( x^*| x)\exp\left(\sum_{\phi \in \Phi}\phi( x)\right), 
q( x| x^*)\exp\left(\sum_{\phi \in \Phi}\phi( x^*)\right) \right)\Bigg)\Bigg]
\end{align*}

We know that $\frac{\min(A, B)}{\min(C,D)} = \min\left(\frac{A}{\min(C,D)}, \frac{B}{\min(C,D)}\right) \geq \min\left(\frac{A}{C}, \frac{B}{D}\right)$. The last inequality is due to the fact that $\frac{1}{\min(C,D)}\geq \frac{1}{C}$ and $\frac{1}{\min(C,D)}\geq \frac{1}{D}$.

With this inequality, we can continue simplifying the ratio,
\begin{align*}
  \frac{\pi( x) T( x^*,  x)}
  {\pi( x) G( x^*,  x)}
&\geq
\mathbf{E}\Bigg[\min \Bigg(\frac{
\exp\left(\sum_{\phi\in \Phi} r_{\phi}\log\left( 1+ \frac{L}{\lambda M_{\phi}}\phi( x) \right)  \right)}
{
\exp\left(\sum_{\phi \in \Phi}\phi( x)\right)},
\\&\hspace{2em}\frac{
\exp\left(\sum_{\phi\in \Phi} r_{\phi}\log\left( 1+ \frac{L}{\lambda M_{\phi}}\phi( x^*)\right)  \right)}
{
\exp\left(\sum_{\phi \in \Phi}\phi( x^*)\right)}
\Bigg)\Bigg]\\
&=
\mathbf{E}\Bigg[\min \Bigg(
\exp\left(\sum_{\phi\in \Phi}\Bigg( r_{\phi}\log\left( 1+ \frac{L}{\lambda M_{\phi}}\phi( x) \right) - \phi( x) \Bigg)\right),
\\&\hspace{2em}\exp\left(\sum_{\phi\in \Phi}\Bigg( r_{\phi}\log\left( 1+ \frac{L}{\lambda M_{\phi}}\phi( x^*)\right) - \phi( x^*) \Bigg)\right)
\Bigg)\Bigg]\\
&=
\mathbf{E}\Bigg[\max \Bigg(
\exp\left(\sum_{\phi\in \Phi} \Bigg(\phi( x) - r_{\phi}\log\left( 1+ \frac{L}{\lambda M_{\phi}}\phi( x) \right)\Bigg)\right),
\\&\hspace{2em}\exp\left(\sum_{\phi\in \Phi} \Bigg(\phi( x^*) - r_{\phi}\log\left( 1+ \frac{L}{\lambda M_{\phi}}\phi( x^*)\right)\Bigg) \right)
\Bigg)^{-1}\Bigg]
\end{align*}

Because $f(x) = \frac{1}{x}$ is a convex function, by Jensen's inequality it follows
\begin{align*}
  \frac{\pi( x) T( x^*,  x)}
  {\pi( x) G( x^*,  x)}
&\geq
\mathbf{E}\Bigg[\max \Bigg(
\exp\left(\sum_{\phi\in \Phi} \Bigg(\phi( x) - r_{\phi}\log\left( 1+ \frac{L}{\lambda M_{\phi}} \phi( x) \right)\Bigg) \right),
\\&\hspace{2em}\exp\left(\sum_{\phi\in \Phi}\Bigg( \phi( x^*) - r_{\phi}\log\left(1+ \frac{L}{\lambda M_{\phi}}\phi( x^*)\right) \Bigg)\right)
\Bigg)\Bigg]^{-1}
\end{align*}
We have that the maximum of the product is less than the product of maximum, therefore
\begin{align*}
  \frac{\pi( x) T( x^*,  x)}
  {\pi( x) G( x^*,  x)}
&\geq
\prod_{\phi\in\Phi}\mathbf{E}\Bigg[\max \Bigg(
\exp \Bigg(\phi( x) - r_{\phi}\log\left( 1+ \frac{L}{\lambda M_{\phi}} \phi( x) \right)\Bigg) ,
\\&\hspace{2em}\exp\Bigg( \phi( x^*) - r_{\phi}\log\left(1+ \frac{L}{\lambda M_{\phi}}\phi( x^*)\right) \Bigg)
\Bigg)\Bigg]^{-1}
\end{align*}
Since $\max(A,B)\leq A + B$ when $A$ and $B$ are positive, it follows
\begin{align*}
  \frac{\pi( x) T( x^*,  x)}
  {\pi( x) G( x^*,  x)}
&\geq
\prod_{\phi\in\Phi}\mathbf{E}\Bigg[
\exp \Bigg(\phi( x) - r_{\phi}\log\left( 1+ \frac{L}{\lambda M_{\phi}} \phi( x) \right)\Bigg) +
\\&\hspace{2em}\exp\Bigg( \phi( x^*) - r_{\phi}\log\left(1+ \frac{L}{\lambda M_{\phi}}\phi( x^*)\right)  \Bigg)\Bigg]^{-1}
\end{align*}
$\mathbf{E}\Bigg[
\exp \Bigg( - r_{\phi}\log\left( 1+ \frac{L}{\lambda M_{\phi}} \phi( x) \right)\Bigg)\Bigg]$ is the moment generating function of the Poisson random variable $r_{\phi}$ evaluated at
\[
t = -\log\left( 1+ \frac{L}{\lambda M_{\phi}} \phi( x) \right)
\]
We know that
\begin{align*}
\mathbf{E}\exp(r_{\phi}t) 
&= 
\exp\left( \frac{\lambda M_{\phi}}{L}\left(\exp(t) - 1\right)\right)\\
\end{align*}
Therefore,
\begin{align*}
\mathbf{E}\Bigg[
\exp \Bigg( - r_{\phi}\log\left( 1+ \frac{L}{\lambda M_{\phi}} \phi( x) \right)\Bigg)\Bigg]
&= 
\exp\left(-\frac{\phi(x)}{1+\frac{L}{\lambda M_{\phi}}\phi(x)}\right)
\end{align*}
Substituting this into the original expression produces
\begin{align*}
  \frac{\pi( x) T( x^*,  x)}
  {\pi( x) G( x^*,  x)}
&\geq
\left[2\prod_{\phi\in\Phi}\exp \Bigg(-\frac{\phi(x)}{1+\frac{L}{\lambda M_{\phi}}\phi(x)} + \phi(x)\Bigg)\right]^{-1}\\
&\geq
\left[2\prod_{\phi\in\Phi}\exp(M_{\phi})\exp \Bigg(-\frac{1}{1+\frac{L}{\lambda }} + 1\Bigg)\right]^{-1}\\
&=
\left[2\prod_{\phi\in\Phi}\exp(M_{\phi})\exp \Bigg(\frac{L}{\lambda + L} \Bigg)\right]^{-1}\\
&=
\left[2\exp \Bigg(\frac{L^2}{\lambda + L} \Bigg)\right]^{-1}\\
&=
\frac{1}{2}\exp \Bigg(-\frac{L^2}{\lambda + L} \Bigg)
\end{align*}

From Dirichlet form argument, we get
  \[
    \bar{\gamma}
    \ge
    \frac{1}{2}\exp\left( - \frac{L^2}{\lambda + L}\right)\cdot\gamma.
  \]
\end{proof}

\subsection{Additional Experiment: Poisson-MH on Truncated Gaussian Mixture}
We test Poisson-MH on the truncated Gaussian mixture as in Section \ref{sec:gmm}. The proposal is $q(x^*|x)=\mathcal{N}(x, 0.45^2 I)$. We set $\lambda=500$. The estimated density is in Figure \ref{fig:mog-pmh} which is very close to the true density. This demonstrates the effectiveness of Poisson-MH and the general applicability of Poisson-minibatching method.

\begin{figure*}[h]
\centering
\includegraphics[width=6.5cm]{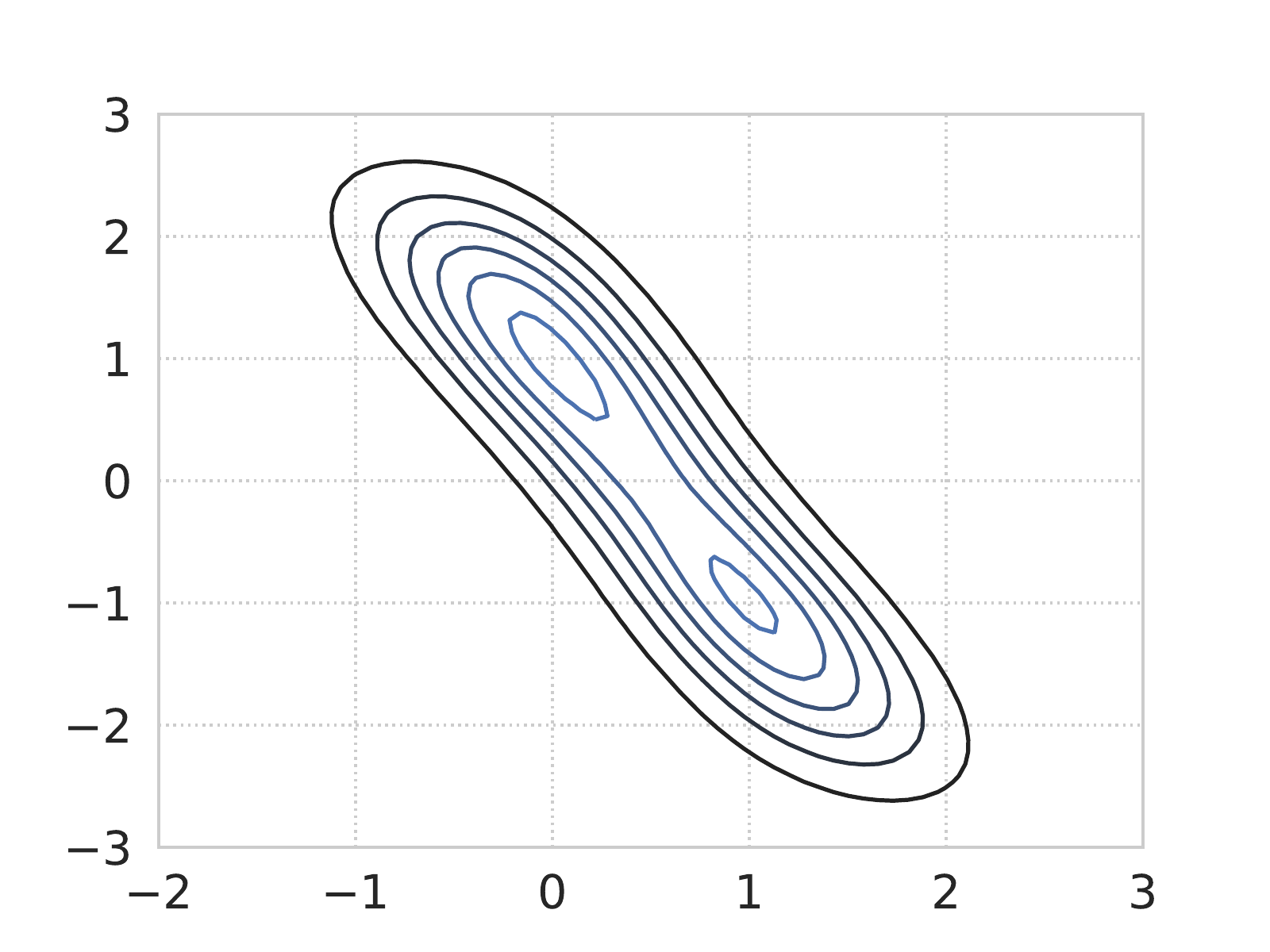}
\caption{The estimated density of Poisson-MH on a truncated Gaussian mixture model.}\label{fig:mog-pmh}
\end{figure*}

\section{Extended Results about Chebyshev Interpolants}\label{sec:chebyext}

In \citet{trefethen2013approximation}, Theorem 8.2 proves bounds on the error of a Chebyshev interpolant on the interval $[-1,1]$.
However, in order to apply this theorem to a second Chebyshev interpolant that is a function of the first, we would need to bound the magnitude of that function \emph{on a Bernstein ellipse}.
To do this, we need the following extended version of Theorem 8.2, which bounds the error not only on the interval $[-1,1]$ but more generally on a Bernstein ellipse.

\begin{lemma}
\label{thm:chebyext}
Assume $U: \mathbb{C} \rightarrow \mathbb{C}$ is analytic in the
open Bernstein ellipse $B([-1, 1], \rho)$, where the Bernstein ellipse is a region in the complex plane bounded by an ellipse with foci at $\pm 1$ and semimajor-plus-semiminor axis length $\rho > 1$. If for all $x \in B([-1, 1], \rho)$, $|U(x)| \leq V$ for some
constant $V > 0$, then for any constant $1 < \sigma < \rho$, the error of the Chebyshev interpolant on the smaller Bernstein ellipse $B([-1,1], \sigma)$ is bounded by
\begin{align*}
| \tilde{U}(x) - U(x) | \leq \frac{4V}{\rho / \sigma - 1} \cdot \left( \frac{\rho}{\sigma} \right)^{-m}.
\end{align*}
\end{lemma}
\begin{proof}
This proof is essentially identical to that of Theorem 8.2 in \citet{trefethen2013approximation}, except that the error is bounded in a Bernstein ellipse rather than over only the real interval $[-1,1]$.

First, note that one parameterization of the boundary of the Bernstein ellipse with parameter $\rho$ is
\[
    \left\{ \frac{z + z^{-1}}{2} \middle| z \in \mathbb{C}, \; \Abs{z} = \rho \right\},
\]
and the open ellipse itself can be written as
\[
    B([-1, 1], \rho) = \left\{ \frac{z + z^{-1}}{2} \middle| z \in \mathbb{C}, \; \rho^{-1} \le \Abs{z} \le \rho \right\}.
\]
Now, Theorem 8.1 from \citet{trefethen2013approximation} states that the Chebyshev coefficients of a function that satisfies the conditions of this theorem (boundedness and analyticity in a Bernstein ellipse) are bounded by $\Abs{a_0} \le V$ and
\[
    \Abs{ a_k } \le 2 V \rho^{-k}, \; k \ge 1.
\]
That is, for $a_k$ bounded in this way,
\[
    U(x) = \sum_{k=0}^{\infty} a_k T_k(x)
\]
at least for all $x$ in the $\rho$-Bernstein ellipse on which $f$ is analytic.
(While \citet{trefethen2013approximation} only states explicitly that this holds for $x \in [-1,1]$, the fact that it also holds on the rest of the Bernstein ellipse follows directly from the fact that both sides of the equation are analytic over that region, using the identity theory for holomorphic functions.)
Formula (4.9) from \citet{trefethen2013approximation} states that
\[
    U(x) - \tilde U_m(x) 
    =
    \sum_{k=m+1}^{\infty} a_k \left( T_k(x) - T_{l(k,m)}(x) \right)
\]
where $\tilde U_m$ denotes the degree-$m$ Chebyshev interpolant, and
\[
    l(k,m) = \Abs{ \left( (k + m - 1) \bmod 2m \right) - (m - 1) }.
\]
Notice in particular that it always holds that $l(k,m) \le m+1$.
Now, for $x$ inside the Bernstein ellipse $B([-1, 1], \sigma)$, there will always exist a $z \in \mathbb{C}$ such that $\sigma^{-1} \le \Abs{z} \le \sigma$ and
\[
    x = \frac{z + z^{-1}}{2}.
\]
For such an $x$, and for any $k$,
\[
   \Abs{T_k(x)} 
   = 
   \Abs{ T_k\left( \frac{z + z^{-1}}{2} \right) }
   = 
   \Abs{ \frac{z^k + z^{-k}}{2} }
   = 
   \frac{\Abs{z}^k + \Abs{z}^{-k}}{2}
   \le
   \sigma^k,
\]
where the second equality is a well-known property of the Chebyshev polynomials.
It follows that, for any $x$ in this Bernstein ellipse,
\begin{align*}
    \Abs{ U(x) - \tilde U_m(x) }
    &=
    \Abs{ \sum_{k=m+1}^{\infty} a_k \left( T_k(x) - T_{l(k,m)}(x) \right) }
    \\ &\le
    \sum_{k=m+1}^{\infty} \Abs{a_k} \cdot \Abs{ T_k(x) - T_{l(k,m)}(x) }
    \\ &\le
    \sum_{k=m+1}^{\infty} 2 V \rho^{-k} \cdot \left( \sigma^k + \sigma^{l(k,m)} \right)
    \\ &\le
    4 V \sum_{k=m+1}^{\infty} \rho^{-k} \sigma^k
    \\ &\le
    4 V \left(\frac{\sigma}{\rho} \right)^{m+1} \sum_{k=0}^{\infty} \left(\frac{\sigma}{\rho} \right)^k
    \\ &\le
    4 V \left(\frac{\sigma}{\rho} \right)^{m+1}  \frac{1}{1 - \frac{\sigma}{\rho}}
    \\ &\le
    4 V \left(\frac{\sigma}{\rho} \right)^{m}  \frac{1}{\rho / \sigma - 1}.
\end{align*}
This is the desired result.
\end{proof}

\end{document}